\documentclass[10pt,journal]{IEEEtran}
\usepackage[margin=1in]{geometry}
\usepackage{amsmath,amssymb,amsthm,mathtools,bm}
\usepackage{hyperref,cleveref,booktabs,array,graphicx}
\usepackage{algorithm,algpseudocode}
\usepackage{caption}
\usepackage{url}

\usepackage{xcolor}

\newtheorem{theorem}{Theorem}

\title{Should I Use This Synthetic Dataset for Training? How to Test with Minimal Real Data}

\author{Zhenyu~Tao,~\IEEEmembership{Graduate Student Member,~IEEE},
  Wei~Xu,~\IEEEmembership{Fellow,~IEEE},
  Xiaohu~You,~\IEEEmembership{Fellow,~IEEE},
  Petar~Popovski,~\IEEEmembership{Fellow,~IEEE},
  Osvaldo~Simeone,~\IEEEmembership{Fellow,~IEEE}
  \thanks{The work of O. Simeone was supported by an Open Fellowship of the EPSRC (EP/W024101/1), by the EPSRC (EP/X011852/1), and by the ERC (No. 101198347). The work of P. Popovski was supported, in part, by the Velux Foundation, Denmark, through the Villum Investigator Grant WATER, nr. 37793.}
  \thanks{Z. Tao, W. Xu, and X. You are with the National Mobile Communications Research Lab, Southeast University, Nanjing 210096, China, and also with the Pervasive Communication Research Center, Purple Mountain Laboratories, Nanjing 211111, China (email: \{zhenyu\_tao, wxu, xhyu\}@seu.edu.cn).}
  \thanks{P. Popovski is with the Department of Electronic Systems, Aalborg University, 9220 Aalborg, Denmark (e-mail: petarp@es.aau.dk).}
  \thanks{O. Simeone is with the Institute for Intelligent Networked Systems, Northeastern University London, E1 8PH London, U.K., and also with the Connectivity Section, Department of Electronic Systems, Aalborg University, 9220 Aalborg, Denmark (e-mail: o.simeone@northeastern.edu).}

}
\date{}

\begin{document}

\maketitle
\begin{abstract}
Digital twins (DTs) and learned world models are
increasingly used to generate synthetic data that augment the scarce
real datasets available for training artificial intelligence (AI) models
in engineering systems. Owing to the inevitable
simulation-to-reality (sim-to-real) gap, however, augmentation may fail
to improve the performance of the trained model on
the real data distribution. This paper addresses the resulting decision
problem: Given a real dataset, a candidate synthetic dataset, and a fixed
learning algorithm, decide whether training on the augmented dataset
improves the true, population-level performance, while consuming as few
real test data points as possible. Two formulations are considered: a
direct test on the mean loss difference between the two trained
models, and a symmetry-based test on the paired loss difference,
which trades a stronger null assumption for faster evidence
accumulation. For the latter, we introduce the {adaptive e-process
sign-flip test} (aeSFT), a doubly adaptive procedure that adapts both the
number of Monte Carlo sign-flip rounds, and hence the computational cost,
and the amount of real test data consumed. aeSFT yields
anytime-valid Type-I error control, with no need to
pre-specify the test-set size. Experiments on a synthetic-data
classification task, a DT-aided wireless packet-scheduling task, and a
radio-map prediction task
show that aeSFT identifies useful synthetic data using substantially
fewer real test samples than mean-based sequential testing, matches the power of
fixed-sample sign-flip testing and the paired $t$-test, while keeping the
false-positive rate below the target level.
\end{abstract}


\section{Introduction}

\subsection{Motivation and Context}
\label{subsec:motivation}

The dominant constraint on the deployment of artificial intelligence (AI)
in engineering systems is rarely model capacity. Rather, it is the availability
of data that reflect the specific system to be controlled. For example, wireless
network performance depends on the propagation environment, topology,
and traffic pattern of a particular deployment. Therefore, a learned control policy is only
as good as the operating conditions it has observed. However, collecting such
deployment-specific measurements is slow, expensive, and sometimes
disruptive to a live system \cite{ruah2026bridge}.

A now-standard response to this constraint is to manufacture data. In
engineering domains, \emph{digital twins} (DTs), i.e., high-fidelity,
site-specific simulators or emulators of the real system, can generate
essentially unlimited synthetic data tailored to the deployment of
interest~\cite{10628026}. Examples include differentiable ray-tracing engines for the
radio propagation environment \cite{hoydis2024learning,hoydis2022sionna},
network emulators built on software-defined radios
\cite{polese2024colosseum}, and DT-driven reinforcement learning
(RL) for network control \cite{11299846}. In parallel, \emph{world
models}, i.e., learned generative models of environment dynamics that
serve as differentiable, queryable surrogates of reality
\cite{ha2018world,hafner2023dreamer,bruce2024genie}, are emerging as key components of intelligent systems. Using DTs or world models, policies can potentially be trained largely, or entirely, using synthetic data~\cite{Hafner2025dreamer}.

\begin{figure*}[t]
\centering
\includegraphics[width=0.95\textwidth]{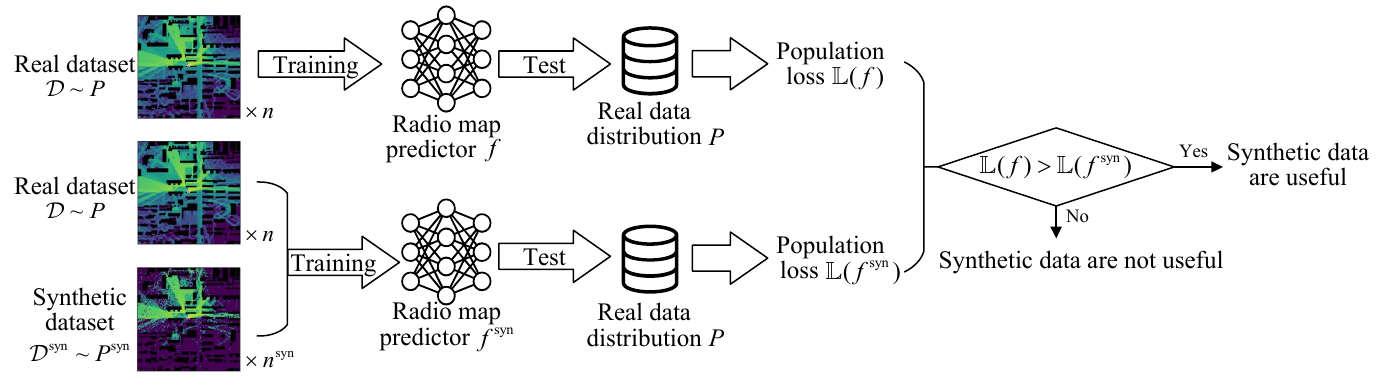}
\caption{\small Illustration of the problem setting: Given a real dataset $\mathcal{D}$ and a synthetic dataset $\mathcal{D}^{\mathrm{syn}}$, we wish to decide whether training on the augmented dataset $\mathcal{D}\cup\mathcal{D}^{\mathrm{syn}}$ reduces the population loss of the trained model.}
\label{fig:problem}
\end{figure*}

However, such models are simplifications of the systems they replicate. For example, in a ray tracer, geometry is approximated, and material parameters are estimated~\cite{Hoydis2023SionnaRT}. The resulting
\emph{simulation-to-reality} (sim-to-real) gap implies that the synthetic
distribution differs from the real one, and hence training on the union of real and synthetic datasets generally optimizes a criterion that is a biased estimate of the target population loss
\cite{ruah2026bridge,tobin2017domain}.

In general, whether a particular synthetic dataset, produced by a given DT or world model, helps a particular learning algorithm on a given task cannot be ascertained based only on the simulator's output. Rather, it requires evaluation using real data. In fact, since the quantity of interest is the population loss under the true data distribution, the only reliable
arbiter is a test dataset drawn from that distribution, but such test data
are expensive, corresponding to episodes on a physical robot \cite{mahmood2018benchmarking}, to queries answered by
human labellers \cite{chiang2024chatbot}, or to runs of a high-fidelity emulator or live network \cite{polese2024colosseum}. Consequently, in practice, it is critical to limit the number of real evaluations needed to establish the usefulness of synthetic data with a statistical guarantee.

\subsection{Related Work}
\label{subsec:related}

\emph{Synthetic data from DTs and world models.} A first line
of work seeks to reduce the sim-to-real gap at its source, by calibrating
the simulator against real measurements, e.g., via differentiable ray
tracing \cite{hoydis2024learning} or phase-error-aware calibration of the
scene geometry \cite{ruah2024calibrating}. A second line accepts a
residual gap and makes training robust to it, either by modelling
uncertainty about the environment in a Bayesian fashion
\cite{ruah2023bayesian}, or by correcting the training objective itself
using a small amount of real data through prediction-powered inference
(PPI) \cite{angelopoulos2023ppi} and its cross-validated
\cite{sifaou2025semi} and context-aware \cite{ruah2025context} variants;
see \cite{ruah2026bridge} for a review. These methods are complementary
to the present work: they aim to make synthetic data more useful, whereas
we ask whether a given synthetic dataset is useful at all, and provide a
statistical certificate for the answer. In particular, none of them
returns a decision on the usefulness of synthetic data with a controlled error probability (with respect to the distribution of the data used for testing the quality of synthetic data).

\emph{Data valuation and selection.} Assessing the contribution of
training data to model performance is the subject of a large literature,
including influence functions \cite{koh2017influence} and Shapley-value
data valuation \cite{ghorbani2019datashapley}. These techniques are
computational and attributional in nature: they estimate a contribution,
typically on held-out data treated as exact, without controlling the
error of the resulting decision. The problem addressed here is
complementary and explicitly statistical: it concerns the number of real
evaluations needed for a decision that is valid under the true
distribution.

\emph{Sequential testing by betting.} Our tools come from the recent
literature on e-values and game-theoretic statistics. In the
testing-by-betting framework \cite{shafer2021testing,ramdas2025evalues},
evidence against a null is represented as the wealth of a gambler playing
a game that is fair under the null. Such methods support optimal continuation and stopping while retaining statistical validity in terms of Type-I error. Betting strategies for means of bounded
random variables were developed in \cite{waudbysmith2024betting}, and
adapted to sequential hyperparameter and risk-control problems in
adaptive learn-then-test \cite{zecchin2025altt}, building on the
fixed-sample learn-then-test framework \cite{angelopoulos2025ltt}. This methodology can be applied to the problem of testing the difference between the population losses of a model trained with both real and synthetic data and of one trained just with real data. This approach yields an \emph{adaptive mean
test} (aMT), which is anytime-valid but can be
slow when the paired loss differences are heavy-tailed.

\emph{Randomization and Monte Carlo tests.} Instead of testing the mean loss difference, reference~\cite{sohm2026improving} recently proposed to test a more specialized hypothesis whereby unhelpful synthetic data implies a symmetric distribution of the loss difference. This hypothesis admits a randomization test~\cite{ritzwoller2025randomization}, which requires the number of random evaluations to be fixed in advance. This limitation was removed in
\cite{fischer2025smc}, which recasts randomized testing itself as a
betting game, and thereby adapts the number of random evaluations to the
difficulty of the instance. Importantly, however, adaptivity there concerns
\emph{computation}, not \emph{data}: the test set is fixed in advance,
which leaves untouched the bottleneck that matters most in our setting.

\subsection{Main Contributions}
\label{subsec:contributions}

This paper studies how to decide, with as few real test data points as
possible, whether augmenting a real training set with a given synthetic
dataset improves real-world performance. The main contributions are as
follows.

\begin{itemize}
\item \textbf{Formulation.} We formalize synthetic data usefulness as a
sequential binary testing problem under two alternative nulls: a
\emph{direct} null, stating that the mean loss difference between the
model trained on real data and the model trained on real plus synthetic
data is non-positive; and a \emph{symmetry-based} null, stating that the
paired sample-level loss difference is distributed symmetrically about
zero. We make explicit the assumptions under which each formulation is
appropriate, and the different guarantees that a rejection provides.

% \item \textbf{Background.} We review the two state-of-the-art procedures addressing the two formulations: aMT based on
% betting on the normalized loss difference \cite{waudbysmith2024betting,
% zecchin2025altt}, and the fixed-sample \cite{sohm2026improving} and
% efficient \cite{fischer2025smc} sign-flip tests. We also identify their
% respective bottlenecks: slow evidence accumulation under heavy-tailed
% differences for the former, and a test-set size that must be fixed in
% advance for the latter.

\item \textbf{Adaptive e-process
sign-flip test.} We introduce the \emph{adaptive e-process
sign-flip test} (aeSFT), which is adaptive along both data and complexity axes. Real test data are processed in independent batches; within a
batch, an \emph{intra-batch} e-process accumulates evidence over an
adaptive number of sign-flip rounds and may close the batch early when it
looks unpromising; across batches, an \emph{inter-batch} e-process
compounds the wealth inherited from previous batches, stopping as soon as enough evidence is collected. In the betting
interpretation, as illustrated in Fig.~\ref{fig:sft-esft-aesft-illustration}, the gambler is granted access to a sequence of casinos,
each with a fresh pool of real test data, and carries the accumulated
wealth from one to the next. 

\begin{figure*}[t]
\centering
\includegraphics[width=\textwidth]{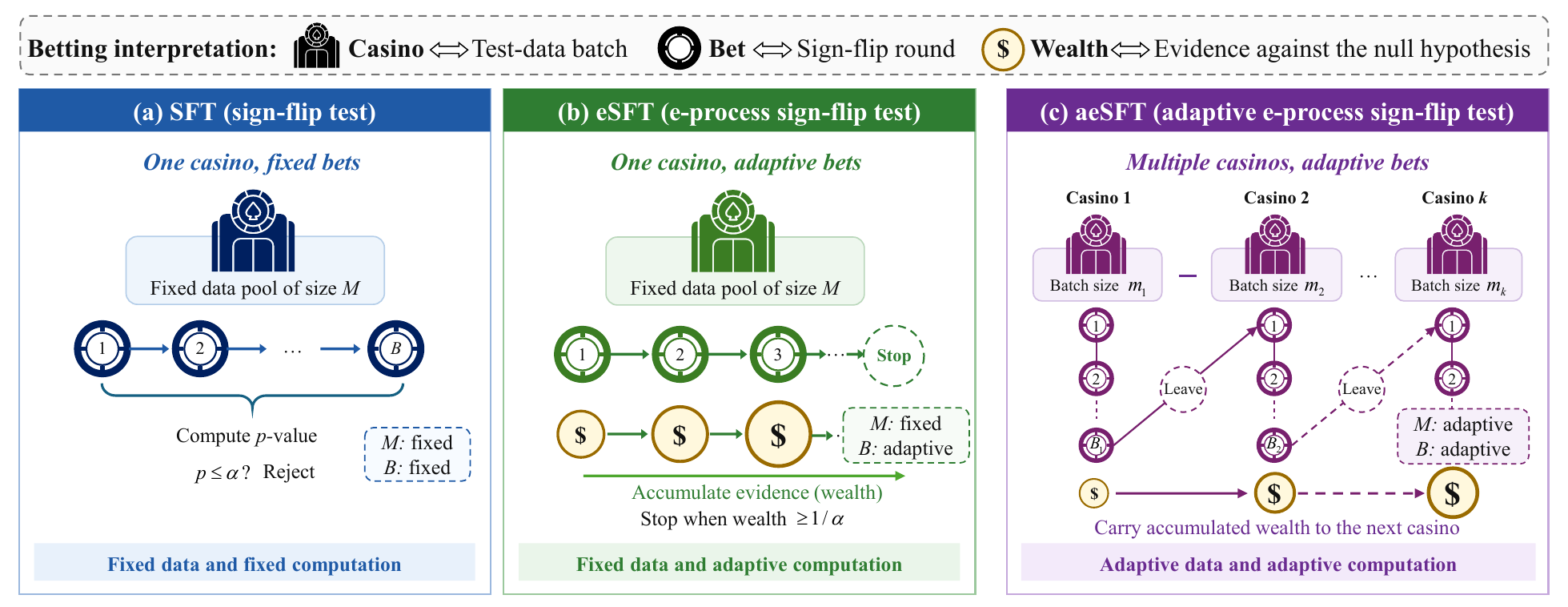}
\caption{\small Betting interpretation of SFT~\cite{sohm2026improving}, eSFT~\cite{fischer2025smc}, and aeSFT (this work).}
\label{fig:sft-esft-aesft-illustration}
\end{figure*}

\item \textbf{Theoretical guarantee.} We prove that the inter-batch
e-process is a nonnegative test martingale under the symmetry-based
null, and that the number of sign-flip rounds is a stopping time
with respect to the intra-batch history. Consequently, by Ville's inequality, rejecting the null
the first time the inter-batch e-process reaches $1/\alpha$ controls the Type-I error at
level $\alpha$ uniformly over time.

\item \textbf{Experiments.} On a binary classification task with
synthetic training data, a DT-aided wireless packet-scheduling
task in which an offline deep $Q$-network is trained on real and
DT-generated trajectories~\cite{11625736}, and a radio-map prediction task~\cite{Lee2023PMNet}, aeSFT attains a markedly higher true-positive rate than aMT~\cite{shafer2021testing} at small real-data consumption, matches the power of
fixed-sample procedures without having to pre-specify a sample size, and
keeps the false-positive rate below the target level $\alpha$.
\end{itemize}

The remainder of the paper is organized as follows.
Sec.~\ref{sec:problem} formulates the two tests and the error-control
requirement. Sec.~\ref{sec:background} and Sec.~\ref{sec:fixed-symmetry} review sequential mean testing by
betting and sign-flip testing. Sec.~\ref{sec:aesft} presents aeSFT and
its guarantee. Sec.~\ref{sec:experiments} reports the experiments, and
Sec.~\ref{sec:conclusion} concludes.

\section{Problem Definition}\label{sec:problem}

Let $A(\cdot)$ denote a fixed learning algorithm that maps a training dataset to a trained model $f$. Given a real training dataset $\mathcal{D}$ of $n$ data points drawn from the ground-truth data distribution $P$, and a synthetic dataset $\mathcal{D}^{\mathrm{syn}}$ of $n^{\mathrm{syn}}$ data points drawn from a generally different distribution $P^{\mathrm{syn}}$, we address the following question:

\begin{center}
\emph{Would augmenting the real dataset $\mathcal{D}$ with the synthetic dataset $\mathcal{D}^{\mathrm{syn}}$ improve the test-time
performance of the trained model?}
\end{center}

\noindent We formulate this question in terms of a scalar population-level loss function $\mathbb{L}(f)$ that should be minimized. This loss may be an expected prediction loss, a shifted negative expected reward, an expected task-completion time, or another task-dependent performance criterion. Let $Z\sim P$ denote a test sample, and let $\mathbb{L}(f;Z)$ denote the sample-level objective function. The sample $Z$ may be an input-label pair, an interaction episode, or another task-dependent evaluation unit. The population objective is the average loss
\begin{equation}
\mathbb{L}(f)=\mathbb{E}_{Z\sim P}\!\left[\mathbb{L}(f;Z)\right].
\label{eq:objective}
\end{equation}
For example, in the prediction case, with sample $Z=(X,Y)$, the sample-level objective is given by
\begin{equation}
\mathbb{L}(f;Z)=\ell(f(X),Y),
\label{eq:supervised-special-case}
\end{equation}
for a given non-negative loss function $\ell(\hat y,y)\geq 0$ between outcome $y$ and prediction $\hat y$.

In order to address the question above, we assume access to independent and identically distributed (i.i.d.) test data
\begin{equation}
Z_1,Z_2,\ldots,Z_m\overset{\mathrm{iid}}{\sim}P,
\label{eq:test-data-stream}
\end{equation}
drawn from the real data distribution $P$.
Acquiring test data may require making real-world measurements on a physical system~\cite{mahmood2018benchmarking}, querying human labelers~\cite{chiang2024chatbot}, or running high-fidelity emulators~\cite{polese2024colosseum}. Accordingly, it is important to use as few test data points, $m$, as possible. To accomplish this goal, we will focus on sequential tests that can stop adaptively as soon as enough evidence is available to make a decision.

\subsection{Direct Synthetic Data Test}\label{sec:direct-test}

The most direct formalization of the problem compares the model trained only on real data, i.e.,
\begin{equation}\label{eq:real-model}
f=A(\mathcal{D}),
\end{equation}
with the model trained on the union of real and synthetic data,
\begin{equation}\label{eq:synthetic-model}
f^{\mathrm{syn}}=A(\mathcal{D}\cup\mathcal{D}^{\mathrm{syn}}).
\end{equation}
Since smaller values of $\mathbb{L}(\cdot)$ are preferable, synthetic data augmentation is useful when the loss of model~\eqref{eq:synthetic-model} is smaller than that of the original model~\eqref{eq:real-model}. Define the population loss gap as $\Delta=\mathbb{L}(f)-\mathbb{L}(f^{\mathrm{syn}})$, so usefulness of synthetic data corresponds to $\Delta>0$.
Accordingly, we can formulate the problem in the form of a binary test for the null hypothesis
\begin{equation}\label{eq:H_0}
H_0:\;\Delta\le 0
\end{equation} 
that the synthetic data are not useful, against the complementary alternative hypothesis
\begin{equation}\label{eq:H_1}
H_1:\;\Delta>0.
\end{equation}
By definition, rejecting the null $H_0$ amounts to concluding that there is sufficient evidence that the synthetic data are useful for training.

\subsection{Symmetry-Based Synthetic Data Test}\label{sec:sym-test}
The direct hypothesis test~\eqref{eq:H_0}--\eqref{eq:H_1} relies on the comparison of mean values of the loss function. This approach can be sample-inefficient when the available test data are limited and the sample-level objective differences have heavy-tailed distributions under the alternative $H_1$, resulting in slow evidence accumulation against the null $H_0$~\cite{sohm2026improving}. To potentially improve testing efficiency in this regime, we also study an alternative formulation of the binary test based on a more restrictive null hypothesis. This formulation may increase the power of the test, supporting the detection of truly useful synthetic datasets with fewer test data points $m$. However, it may also produce more false positives with respect to the original direct test~\eqref{eq:H_0}--\eqref{eq:H_1}.

For any test sample $Z\sim P$, define the paired sample-level difference
\begin{equation}\label{eq:paired-difference}
\delta
=
\mathbb{L}(f;Z)
-
\mathbb{L}(f^{\mathrm{syn}};Z),
\end{equation}
whose expectation is the population loss gap
\begin{equation}
\mathbb{E}[\delta]=\Delta.
\label{eq:sym-delta}
\end{equation}
In contrast to Section~\ref{sec:direct-test}, we make two additional assumptions about the distribution of the difference~\eqref{eq:paired-difference}:

\emph{Assumption 1. Synthetic data cannot harm performance}:
The baseline model $f$, trained without synthetic data, cannot perform better than the model $f^{\mathrm{syn}}$ on average, i.e., $\Delta\geq 0$, so that the condition that the synthetic data are not useful amounts to the equality $\Delta=0$. In practice, this is approximately true if the synthetic data are sufficiently close in distribution to real data.

\emph{Assumption 2. Symmetric distribution under the null}: The condition $\Delta=0$, representing unhelpful synthetic data, is further specialized by assuming that the distribution of the paired difference $\delta$ is symmetric around zero. This assumption is reasonable when synthetic data are equally likely to help and to harm when they are not useful on average. Under this assumption, the null can be expressed as
\begin{equation}
H_0^{\mathrm{syn}}:\;\delta\overset{d}{=}-\delta,
\label{eq:H0-prime}
\end{equation}
where $\overset{d}{=}$ denotes equality in distribution. The corresponding alternative hypothesis is the complement
\begin{equation}
H_1^{\mathrm{syn}}:\;\delta\overset{d}{\neq}-\delta.
\label{eq:H1-prime}
\end{equation}
Accordingly, rejection of the specialized null $H_0^{\mathrm{syn}}$ supports the conclusion that synthetic augmentation improves over training on real data alone.

To summarize, the problem formulation in terms of hypotheses~\eqref{eq:H0-prime}--\eqref{eq:H1-prime} rests on a stronger assumption than the original direct test~\eqref{eq:H_0}--\eqref{eq:H_1}: when synthetic data are not helpful, the paired loss difference $\delta$ in~\eqref{eq:paired-difference} is symmetric around its zero mean. Rejecting the null $H_0^{\mathrm{syn}}$ in~\eqref{eq:H0-prime}--\eqref{eq:H1-prime} does not entail a rejection of the original null $H_0$ in~\eqref{eq:H_0}--\eqref{eq:H_1}, although it provides a strong signal for the usefulness of synthetic data when this assumption is approximately valid.

\subsection{Error Control}
We focus on sequential testing procedures that process test data $Z_1,Z_2,\ldots$ one by one, mapping the available data $Z^i=(Z_1,\ldots,Z_i)$ into a decision
\begin{equation}
 D(Z^i)\in\{\text{reject }H_0,\text{continue}\}.
 \label{eq:sequential-decision}
\end{equation}
We denote the stopping time at which a final decision is made as
\begin{equation}
 m_{\mathrm{stop}}=\min\left\{\min\left\{i:D(Z^i)=\text{reject }H_0\right\},M\right\},
 \label{eq:stopping-time}
\end{equation}
where $M$ is the maximum amount of data that can be collected, e.g., due to time constraints. If $D(Z^{m_{\mathrm{stop}}})=\text{continue}$, the test concludes that the null cannot be rejected. For both the direct test~\eqref{eq:H_0}--\eqref{eq:H_1} and the symmetry-based test~\eqref{eq:H0-prime}--\eqref{eq:H1-prime}, we wish the test to have a Type-I error probability no larger than a user-defined reliability level $\alpha\in(0,1)$ as
\begin{equation}
\Pr\nolimits_{H_0}\left[D(Z^{m_{\mathrm{stop}}})=\text{reject }H_0\right]\le \alpha,
\label{eq:type-I}
\end{equation}
where $\Pr\nolimits_{H_0}[\cdot]$ denotes the worst-case, i.e., largest, probability among all distributions covered by hypothesis $H_0$, and $H_0$ is replaced by $H_0^{\mathrm{syn}}$ for the symmetry-based test~\eqref{eq:H0-prime}--\eqref{eq:H1-prime}. Thanks to the constraint~\eqref{eq:type-I}, the design can ensure that the probability of incorrectly declaring synthetic data useful is maintained below level $\alpha$, where the probability is with respect to the data used for the test.

Subject to the constraint~\eqref{eq:type-I}, we wish to maximize the sample efficiency of the testing procedure by minimizing the number of test data points $m_{\mathrm{stop}}$ required to reject the null when synthetic data are actually useful for improving model performance, i.e., under the alternative $H_1$ or $H_1^{\mathrm{syn}}$.

\begin{table}[t]
\centering
\caption{Notation used throughout the paper.}
\label{tab:notation}
\small
\begin{tabular}{>{\raggedright\arraybackslash}p{0.15\linewidth}p{0.75\linewidth}}
\toprule
Symbol & Meaning \\
\midrule
$P$, $P^{\mathrm{syn}}$ & Real and synthetic data distributions. \\
$\mathcal{D}$, $\mathcal{D}^{\mathrm{syn}}$ & Real and synthetic training datasets. \\
$f$, $f^{\mathrm{syn}}$ & Models trained on $\mathcal{D}$ and $\mathcal{D}\cup\mathcal{D}^{\mathrm{syn}}$. \\
$\mathbb{L}(f;Z)$ & Sample-level loss. \\
$\mathbb{L}(f)$ & Population loss. \\
$\delta$, $\Delta$ & Paired sample-level loss difference and its population mean. \\
$H_0,H_1$ & Direct mean hypotheses based on $\Delta$. \\
$H_0^{\mathrm{syn}},H_1^{\mathrm{syn}}$ & Symmetry-based hypotheses based on $\delta$. \\
$M$, $m_k$ & Maximum test-data budget and size of adaptive batch $k$. \\
$B$, $B_k$ & Fixed and adaptive numbers of sign-flip rounds. \\
$W_k^b$, $\widetilde W_k^b$ & Intra-batch and inter-batch wealth. \\
$\alpha$ & Target Type-I error level. \\
$m_{\mathrm{init}},\beta$ & Initial batch size and batch-growth factor. \\
$\omega,\epsilon$ & Early-stopping threshold and change tolerance. \\
\bottomrule
\end{tabular}
\end{table}

\section{Adaptive Direct Synthetic Data Test}\label{sec:background}

In this section, we address the direct binary test~\eqref{eq:H_0}--\eqref{eq:H_1} by reviewing a state-of-the-art sequential testing strategy based on testing-by-betting~\cite{shafer2021testing,waudbysmith2024betting}. For any test sample $Z\sim P$, define the loss difference $\delta$ as in~\eqref{eq:paired-difference}, so that the null~\eqref{eq:H_0} can be equivalently expressed as $H_0:\mathbb{E}[\delta]\leq 0$. Assume that the sample-level loss is bounded as $0\leq\mathbb{L}(f';Z)\leq L_{\max}$ for some known constant $L_{\max}>0$ and any model $f'$. It follows that $\delta\in[-L_{\max},L_{\max}]$, so the loss difference can be normalized as
\begin{equation}
\widetilde{\delta}=\frac{\delta}{2L_{\max}}+\frac{1}{2}\in[0,1],
\label{eq:direct-risk}
\end{equation}
and the null hypothesis can be equivalently expressed as $H_0:\mathbb{E}[\widetilde{\delta}]\leq 0.5$.

The process can be interpreted as the wealth of a bettor who invests a fraction of the current wealth to bet against the null hypothesis $H_0$~\cite{shafer2021testing,waudbysmith2024betting}, as shown in~Fig~\ref{fig:sft-esft-aesft-illustration}; a larger value of the current wealth provides more evidence against the null. Under aMT~\cite{waudbysmith2024betting,zecchin2025altt}, after observing each test sample $Z_i\sim P$, the wealth statistic $E_i$ is updated as
\begin{equation}\label{eq:amt-eprocess}
E_i=E_{i-1}\cdot \bigl(1+\mu_i(\widetilde{\delta}_i-0.5)\bigr),
\end{equation}
starting from $E_0=1$, where $\mu_i$ is interpreted as the fraction of the wealth $E_{i-1}$ invested against the null $H_0:\mathbb{E}[\widetilde{\delta}]\leq 0.5$. By~\eqref{eq:amt-eprocess}, when $\mathbb{E}[\widetilde{\delta}]\leq 0.5$, the wealth decreases in expectation, while it increases otherwise. The fraction $\mu_i\in(0,1/(1-\alpha))$ can be selected using the past data $Z^{i-1}=(Z_1,\ldots,Z_{i-1})$. The null $H_0$ is rejected once the statistic crosses a threshold $\gamma$, i.e.,
\begin{equation}
E_i\geq\gamma\;\Longrightarrow\;\text{reject }H_0,
\label{eq:amt-reject}
\end{equation}
and otherwise testing continues. If no rejection has occurred up to and including time $M$, the hypothesis $H_0$ cannot be rejected.

\begin{figure*}[t]
\centering
\includegraphics[width=0.95\textwidth]{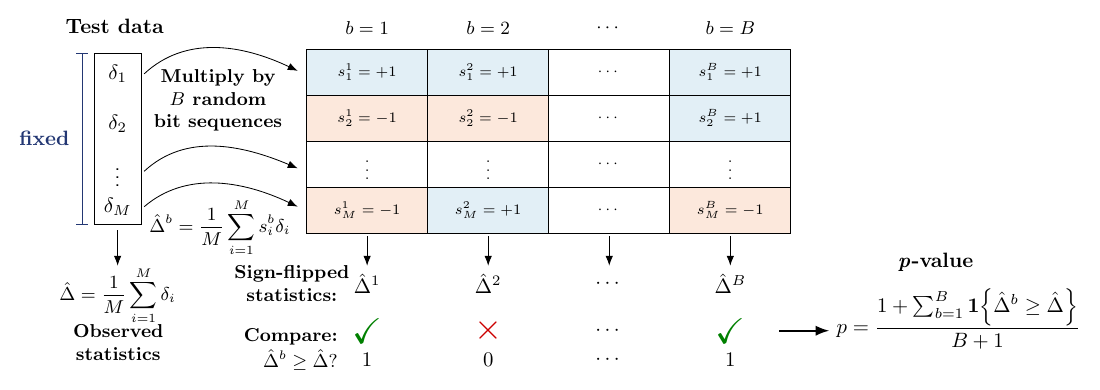}
\caption{\small Illustration of SFT~\cite{sohm2026improving}, where the test-set size $M$ and the number of sign-flip sequences $B$ are fixed in advance.}
\label{fig:fixed-signflip}
\end{figure*}

Reference~\cite{waudbysmith2024betting} shows that the statistic~\eqref{eq:amt-eprocess} forms an e-process, a generalization of sequential likelihood ratios~\cite{wald1947sequential}. The key property of an e-process is that, under the null, its expected next value conditioned on the past cannot exceed its current value. By an extension of the Markov inequality known as Ville's inequality, this implies that choosing the threshold $\gamma=1/\alpha$ in \eqref{eq:amt-reject} ensures that the Type-I error requirement~\eqref{eq:type-I} is satisfied. However, as discussed in Section~\ref{sec:sym-test}, this test may be inefficient, requiring a large number of data points $m_{\mathrm{stop}}$ to reject the null $H_0$ under the alternative $H_1$. This is the case when the paired loss differences $\delta$ are heavy-tailed under the alternative, causing evidence against the null to accumulate slowly.

\section{Fixed-Sample Symmetry-Based Synthetic Data Test}\label{sec:fixed-symmetry}

In this section, we address the symmetry-based test~\eqref{eq:H0-prime}--\eqref{eq:H1-prime} in a fixed-sample, i.e., non-sequential, setting in which the number of test samples $M$ is fixed. We first review the basic sign-flip test (SFT) used in~\cite{sohm2026improving}, and then adapt the more computationally efficient approach presented in~\cite{fischer2025smc} to the current task.

\subsection{Sign-Flip Test}\label{sec:sign-flip}
Because the test samples $Z_1,\ldots,Z_M$ are i.i.d.\ as specified in~\eqref{eq:test-data-stream}, under the null $H_0^{\mathrm{syn}}$ in~\eqref{eq:H0-prime} the paired differences have a joint distribution that is sign-flip invariant in the sense that
\begin{equation}
(\delta_1,\dots,\delta_{M})\overset{d}{=}(s_1\delta_1,\dots,s_{M}\delta_{M}),
\label{eq:sign-flip-invariance}
\end{equation}
where $\delta_i = \mathbb{L}(f;Z_i)-\mathbb{L}(f^{\mathrm{syn}};Z_i)$, for any fixed binary sequence $(s_1,\dots,s_M)\in\{-1,+1\}^M$.

Since the set of all multiplications by binary sequences is a group, one can directly apply the randomization test~\cite{ritzwoller2025randomization}. To this end, define the empirical average
\begin{equation}
\hat\Delta=\frac{1}{M}\sum_{i=1}^{M}\delta_i.
\end{equation}
As illustrated in Fig.~\ref{fig:fixed-signflip}, using the Monte Carlo version of the test, we draw $B$ independent binary sequences
\begin{equation}\label{eq:sign-flips}
\bm{s}^{b}=(s_1^{b},\dots,s_{M}^{b}),\qquad b=1,\dots,B,
\end{equation}
with independent random variables $s_i^{b}\sim\mathrm{Unif}\{-1,+1\}$,
and compute the corresponding sign-flipped statistics
\begin{equation}
\hat\Delta^{b}=\frac{1}{M}\sum_{i=1}^{M}s_i^{b}\delta_i.
\end{equation}
The quantity
\begin{equation}
p=\frac{1+\sum_{b=1}^{B}\mathbf{1}\{\hat\Delta^{b}\ge \hat\Delta\}}{B+1}
\label{eq:p-value}
\end{equation}
is then evaluated.
A small value of the \emph{p}-value \eqref{eq:p-value} means that few sign-flipped statistics produce an improvement as large as the observed one, i.e., $\hat\Delta$, providing strong evidence for the usefulness of synthetic data.
The quantity \eqref{eq:p-value} is a one-sided \emph{p}-value for the null in \eqref{eq:H0-prime}, so that the test
\begin{equation}
p\le\alpha\;\Longrightarrow\; \text{reject }H_0^{\mathrm{syn}}
\label{eq:fixed-reject}
\end{equation}
satisfies the Type-I error probability requirement $\Pr_{H_0^{\mathrm{syn}}}(\text{reject }H_0^{\mathrm{syn}})\leq\alpha$ under the null $H_0^{\mathrm{syn}}$ for any fixed number $B$ of random sign flips, where the probability is also taken over the random binary sequences.

\begin{figure*}[t]
\centering
\includegraphics[width=0.85\textwidth]{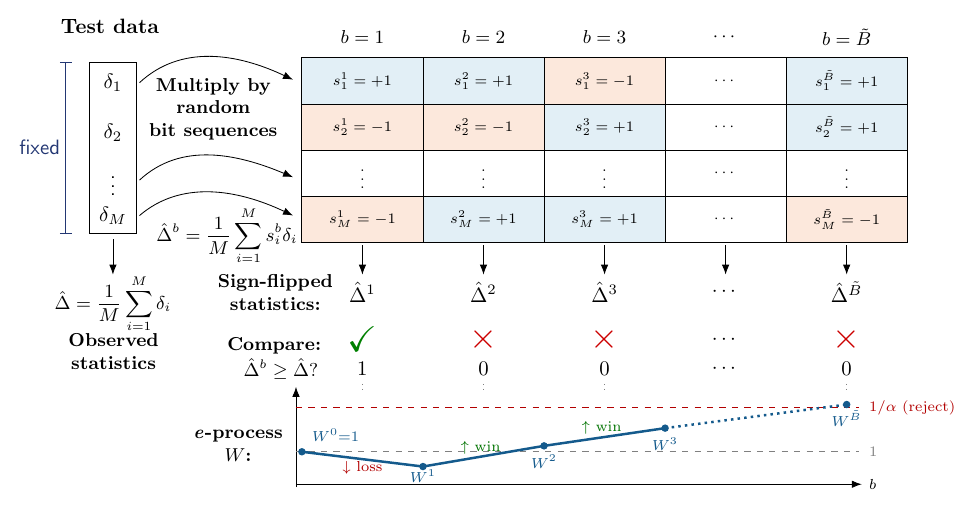}
\caption{\small Illustration of eSFT~\cite{fischer2025smc}: the process stops and rejects the null at the first step $b$ for which the e-process satisfies $W^b \ge 1/\alpha$. The size of the test dataset, $M$, is fixed in advance.}
\label{fig:fr-sequential}
\end{figure*}

\subsection{E-Process Sign-Flip Test (eSFT)}\label{Fischer_test}

A limitation of SFT \eqref{eq:sign-flips}--\eqref{eq:fixed-reject} is that the sign-flip budget $B$ must be chosen in advance. An excessively small number $B$
reduces power, while a larger $B$ increases computational complexity. Moreover, increasing the number of sign-flip sequences $B$ after inspecting the $p$-value \eqref{eq:p-value} in the hope of increasing the power would result in $p$-hacking, compromising the Type-I error guarantee \eqref{eq:type-I}~\cite{ramdas2025evalues}.

To address this problem, reference~\cite{fischer2025smc} proposes to operate in rounds $b=1,2,\dots$, drawing a fresh sign vector $\bm{s}^{b}$ in \eqref{eq:sign-flips} and computing the corresponding sign-flipped statistic $\hat\Delta^{b}$ at each round $b$. For round $b$, define the indicator term in the \emph{p}-value \eqref{eq:p-value} as
\begin{equation}\label{eq:indicator}
I^b=\mathbf{1}\{\hat\Delta^{b}\ge \hat\Delta\}.
\end{equation}
Here, $I^b=1$ means that the sign-flipped statistic beats the observed statistic and is therefore a losing round for a bet against the null; $I^b=0$ is accordingly a winning round.
The sum in \eqref{eq:p-value} up to round $b$ is thus written as
\begin{equation}\label{eq:sum-of-losses}
L^b=\sum_{j=1}^{b} I^j.
\end{equation}
Rather than relying on the \emph{p}-value \eqref{eq:p-value}, the work~\cite{fischer2025smc} uses the sum \eqref{eq:sum-of-losses} to construct an e-process.

Formally, at each step $b$, before observing the indicator $I^b$ in \eqref{eq:indicator}, fix a non-negative betting function $G^b:\{0,1\}\to\mathbb{R}_{\ge 0}$ that is measurable with respect to $\sigma(I^1,\dots,I^{b-1})$, under the constraint
\begin{equation}
G^b(0)\frac{b-L^{b-1}}{b+1}
+
G^b(1)\frac{L^{b-1}+1}{b+1}
=1.
\label{eq:budget-constraint}
\end{equation}
The betting function $G^b$ is the gambler's strategy, specifying the multiplicative payoff $G^b(I^b)$ applied to the current wealth, while the constraint \eqref{eq:budget-constraint} enforces a \emph{fair game} under the null $H_0^{\mathrm{syn}}$ in~\eqref{eq:H0-prime}, $\mathbb{E}_{H_0^{\mathrm{syn}}}[G^b(I^b)\mid I^1,\dots,I^{b-1}]=1$, so wealth does not grow in expectation. Having observed the current indicator $I^b$, the e-process, i.e., the gambler's wealth growth process, is updated as
\begin{equation}\label{eq:e-process}
W^b=W^{b-1}\,G^b(I^b).
\end{equation}
Since the game is fair under $H_0^{\mathrm{syn}}$, large wealth is unlikely when the null holds, so a large value of $W^b$ constitutes strong evidence against $H_0^{\mathrm{syn}}$.

eSFT rejects the null~\eqref{eq:H0-prime} as soon as the e-process $W^b$ crosses the threshold $1/\alpha$, i.e.,
\begin{equation}
W^b\ge \frac{1}{\alpha}\;\Longrightarrow\; \text{reject }H_0^{\mathrm{syn}}.
\label{eq:e-process-reject}
\end{equation}
This rejection rule satisfies the Type-I error requirement \eqref{eq:type-I} under $H_0^{\mathrm{syn}}$~\cite{fischer2025smc}. In the betting interpretation, the gambler continues playing until the wealth multiplies to the target $1/\alpha$, at which point the procedure stops and rejects the null. The procedure is illustrated in Fig.~\ref{fig:fr-sequential}.

Reference~\cite{fischer2025smc} also provides a mechanism to optimize the betting function $G^b$ in order to maximize the test power, i.e., to choose the gambler's strategy that grows the wealth as fast as possible when the alternative $H_1^{\mathrm{syn}}$ holds. Let
\begin{equation}
\eta=\Pr\nolimits_P(I^b=1\mid I^1,\dots,I^{b-1}),
\end{equation}
denote the probability that the current sign flip yields $\hat\Delta^{b}\ge\hat\Delta$ given the past indicators. An optimized betting function $G^b_\star$, maximizing the average log-growth of the e-process \eqref{eq:e-process} under the alternative $H_1^{\mathrm{syn}}$, is given by
\begin{equation}\label{log-opti-bet}
G^b_\star(0)=(1-\eta)\frac{b+1}{b-L^{b-1}},
\qquad
G^b_\star(1)=\eta\frac{b+1}{L^{b-1}+1}.
\end{equation}
This optimized betting function satisfies the budget constraint~\eqref{eq:budget-constraint} by construction.

Since the probability $\eta$ is unknown in practice, reference~\cite{fischer2025smc} proposes to replace it with the deterministic constant
\begin{equation}\label{eta-cal}
\hat{\eta} = \frac{1}{\bigl\lceil \sqrt{2\pi e^{1/6}}/\alpha \bigr\rceil}.
\end{equation}
With this choice, it can be proved that, if the strategy \eqref{log-opti-bet} with $\eta=\hat\eta$ is employed, whenever the \emph{p}-value calculated with a fixed number of rounds $B$ via \eqref{eq:p-value} satisfies $p\le\hat\eta$, there exists a round $\tilde B\le B$ such that $W^{\tilde B}\ge 1/\alpha$. This makes the setting \eqref{eta-cal} a theoretically grounded choice of the parameter $\hat\eta$.

\section{Adaptive Symmetry-Based Synthetic Data Test}
\label{sec:aesft}

The direct aMT formulation requires only the mean null $\mathbb{E}[\delta]\leq0$, but it also requires a known finite bound $L_{\max}$ on the sample-level loss. In contrast, SFT, eSFT, and aeSFT do not require such a bound and can therefore be applied to unbounded losses or returns; their validity instead requires i.i.d.\ paired differences that are symmetric under $H_0^{\mathrm{syn}}$. Interpreting rejection as evidence that synthetic augmentation is useful additionally relies on the assumption described in Section~\ref{sec:sym-test} that synthetic data cannot increase the population loss. Thus, the symmetry-based route trades a boundedness assumption for stronger distributional and interpretive assumptions rather than strictly dominating the direct test.

SFT \eqref{eq:sign-flips}--\eqref{eq:fixed-reject} and eSFT \eqref{eq:budget-constraint}--\eqref{eq:e-process-reject} can be directly applied to the problem of testing whether synthetic data are useful under the symmetry assumption. SFT fixes both the test sample size $M$ and the number of sign-flip sequences $B$. eSFT adapts the number of sign-flip sequences $B$, and thus the computational complexity, to the current input, but it still requires a fixed test dataset size $M$.

In practice, the main bottleneck for the test is the amount of test data required, since the test dataset $\mathcal{D}^{\mathrm{test}}$ must follow, by assumption, the ground-truth distribution $P$, while eSFT cannot address this issue. Specifically, through the betting analogy, eSFT can be regarded as a \emph{single} casino with a fixed data pool of size $M$, where the gambler must keep betting until the wealth reaches $1/\alpha$. In this analogy, a casino is a test-data batch, the stake is the current wealth, and a bet is one sign-flip round. With the data budget fixed in advance, eSFT can neither draw more data to gather additional evidence for a difficult test, nor save data through adaptive design when fewer samples would already suffice.

To address this limitation, we introduce aeSFT, which adapts both the number of sign-flip sequences, and thus the computational complexity, and the amount of test data consumed. Adopting again the betting analogy, the basic idea is to grant the gambler access to a \emph{sequence} of casinos with independent test data batches: when the current batch looks unpromising, the gambler stops early and carries the accumulated wealth to a new casino with fresh i.i.d.\ data, and rejects the null $H_0^{\mathrm{syn}}$ once the total wealth accumulated across casinos reaches $1/\alpha$. This test is formalized through two e-processes: an \emph{intra-batch} e-process that captures the wealth growth within each individual casino, and an \emph{inter-batch} e-process that aggregates the wealth growth across casinos.

\subsection{Intra-Batch E-Process}\label{sec:intra-batch}
To enable the adaptive use of test data, we process test data in sequential batches of size $m_k$ for $k=1,2,\dots$ In batch $k$, we draw a fresh non-overlapping batch
\begin{equation}
\mathcal{D}_k^{\mathrm{test}}=\{Z_{k,i}\}_{i=1}^{m_k}\overset{\mathrm{iid}}{\sim} P,
\end{equation}
in an i.i.d.\ manner from the ground-truth distribution $P$, independently of previous batches. The paired objective difference for sample $i$ within batch $k$ is given by
\begin{equation}
\delta_{k,i}=\mathbb{L}(f;Z_{k,i})
-\mathbb{L}(f^{\mathrm{syn}};Z_{k,i}),
\end{equation}
and the corresponding observed batch statistic for batch $k$ is
\begin{equation}
\hat\Delta_k=\frac{1}{m_k}\sum_{i=1}^{m_k}\delta_{k,i}.
\end{equation}

Similar to eSFT, at each inner step $b$, we evaluate the sign-flipped statistic
\begin{equation}
\hat\Delta_{k}^{b}=\frac{1}{m_k}\sum_{i=1}^{m_k}s_{k,i}^{b}\delta_{k,i},
\end{equation}
with a new sign-flip sequence $\bm{s}_k^b$ consisting of independently generated signs $\{s_{k,i}^{b}\}_{i=1}^{m_k}\overset{\mathrm{iid}}{\sim}\mathrm{Unif}\{-1,+1\}$. As in eSFT, define the indicator $I_k^b=\mathbf{1}\{\hat\Delta_{k}^{b}\ge\hat\Delta_k\}$ and the cumulative count $L_k^b=\sum_{j=1}^{b}I_k^j$. In a manner similar to \eqref{eq:e-process}, these quantities are used to update the intra-batch e-process
\begin{equation}\label{eq:intra-batch-e-process}
     W_k^b=W_{k}^{b-1}\,G_k^b(I_k^b).
\end{equation}
Following \eqref{log-opti-bet}, the betting function $G_k^b(\cdot)$ is selected as
\begin{equation}
G_k^b(0)=(1-\hat\eta)\frac{b+1}{b-L_k^{b-1}},
\qquad
G_k^b(1)=\hat\eta\,\frac{b+1}{L_k^{b-1}+1},
\end{equation}
with $\hat\eta$ given in \eqref{eta-cal}. Starting from a unit stake $W_k^0=1$, the intra-batch e-process $W_k^b$ records the factor by which the gambler's wealth has been multiplied over the first $b$ bets in casino $k$.

Within each batch $k$, the number of sign-flip rounds $B_k$ is chosen adaptively. Since aeSFT allows the gambler to move to a fresh casino, the procedure can stop early within a batch without rejecting the null if the current batch becomes unpromising. Specifically, after observing the outcomes up to round $b$, the decision of whether to close the current batch may follow any criterion that depends only on the information collected so far, i.e., $\{I_k^1,\dots,I_k^{b}\}$, in order to preserve the e-process property. For example, the procedure could stop once $W_k^b \ge 1/\alpha$ or $W_k^b\le \alpha$. The next batch size $m_{k+1}$ is selected before drawing the next test batch. The specific intra-batch stopping rules and batch-size schedule adopted in aeSFT are detailed in Section~\ref{sec:schedule}.

\subsection{Inter-Batch E-Process}\label{sec:inter-batch}

\begin{figure*}[t]
\centering
\includegraphics[width=0.9\textwidth]{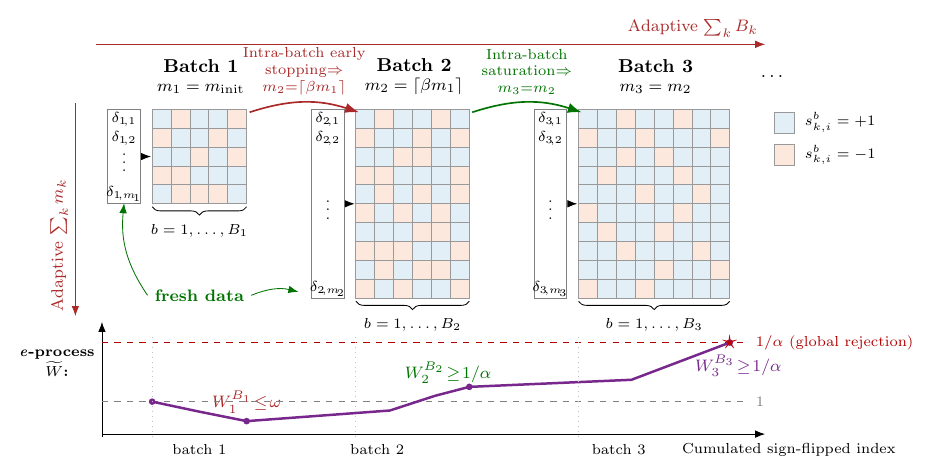}
\caption{\small Illustration of aeSFT. Within each batch $k$, the test uses a fresh test set $\mathcal{D}_k^{\mathrm{test}}$ of size $m_k$ together with an adaptive number $B_k$ of sign-flip rounds.}
\label{fig:adaptive-syn}
\end{figure*}

Let $W_{k}^{B_k}$ denote the final value of the intra-batch e-process for batch $k$. Within each batch, $b=1,\ldots,B_k$, and the pairs $(k,b)$ are ordered lexicographically. The inter-batch e-process is then defined as the product of values of all intra-batch e-processes that have been tested, as depicted in Fig.~\ref{fig:adaptive-syn}, with the e-value after $b$ sign-flip rounds of batch $k$ given by
\begin{equation}\label{eq:inter-batch-e-process}
\widetilde{W}_{k}^{b} = W_{k}^{b}  \times  \prod_{j=1}^{k-1} W_{j}^{B_j},
\end{equation}
starting from $\widetilde{W}_{1}^{0}=1$. In the betting picture, $\widetilde{W}_{k}^{b}$ is the gambler's total wealth. The factor $\prod_{j=1}^{k-1} W_{j}^{B_j}$ is the wealth inherited from previously visited casinos, serving as the starting stake in casino $k$; it is then multiplied by the current within-batch wealth growth $W_k^b$, so that evidence accumulates across batches. The product is natural here because it represents sequential reinvestment of wealth across fresh batches and preserves anytime-valid continuation under predictable adaptation~\cite{ramdas2025evalues}. This test is an e-process under $H_0^{\mathrm{syn}}$, as established in the following theorem, and satisfies the Type-I requirement in \eqref{eq:type-I} while remaining adaptive in both the number of sign-flip sequences $\sum_k B_k$ and the test dataset size $\sum_k m_k$.

\begin{theorem}\label{prop:eprocess}
Under the null $H_0^{\mathrm{syn}}$ in \eqref{eq:H0-prime}, assume that each test batch is independent of previous batches, that $m_k$ is chosen before drawing batch $k$, and that $B_k$ is a stopping time with respect to the intra-batch sign-flip history. Then the wealth growth process $\widetilde{W}_{k}^{b}$ is an e-process. Rejecting the null $H_0^{\mathrm{syn}}$ at the first $(k,b)$ for which $\widetilde{W}_{k}^{b}\ge 1/\alpha$ controls the Type-I error at level $\alpha$, i.e.,
\begin{equation}
\Pr\nolimits_{H_0^{\mathrm{syn}}}\Bigl(\exists(k,b):\;\widetilde{W}_{k}^{b}\ge 1/\alpha\Bigr)\le\alpha.
\end{equation}
\end{theorem}
\begin{proof}
The proof is provided in Appendix~\ref{app:proof}.
\end{proof}

\subsection{Adaptive Budget Schedule}\label{sec:schedule}

The inter-batch e-process of Section~\ref{sec:inter-batch} renders both the total number of sign-flip rounds $\sum_k B_k$ and the total test dataset size $\sum_k m_k$ adaptive. Building on this, we specify the intra-batch stopping mechanism introduced in Section~\ref{sec:intra-batch} and the corresponding batch-size schedule as follows.

\paragraph{Sign-flip round budget}
In eSFT~\cite{fischer2025smc}, the test set is fixed, so the gambler faces a \emph{single} casino and has to continue betting unless the wealth reaches $1/\alpha$. Since further batches are always available in aeSFT, the intra-batch stopping mechanism of Section~\ref{sec:intra-batch} can instead close the current batch without rejecting the null and continue testing with fresh data.

On the early stopping side, persisting on an unpromising batch wastes sign-flip rounds and drives $W_k^{B_k}$ down, so that more subsequent data are needed for $\widetilde{W}_{k}^{b}$ to recover. As the gambler can move to a fresh casino, we therefore declare \emph{intra-batch early stopping} at a threshold $\omega>\alpha$, in order to cut losses before a batch reduces wealth by a factor as severe as $\alpha$. We also require $\omega<1$ so that early stopping is triggered only after the within-batch wealth has fallen below its unit starting value.

On the rejection side, the hard ceiling $W_k^b\ge 1/\alpha$ is tailored to the single-casino setting because it stops as soon as the null can be rejected. However, in aeSFT, the gambler may enter casino $k$ with inherited wealth below 1, so stopping as soon as $W_k^b\ge 1/\alpha$ is not sample efficient when the batch can still contribute fresh evidence for rejection. We therefore replace the hard ceiling with a \emph{change-based} criterion governed by a tolerance $\epsilon>0$: the batch is closed for \emph{intra-batch rejection} only once consecutive winning rounds no longer increase the wealth appreciably, indicating that the evidence in the data batch has been largely exhausted.

Formally, batch $k$ stops at the smallest round $b$ satisfying one of the following conditions:
\begin{equation}\label{eq:intra-stop}
\begin{cases}
b\geq 2,\;\; I_k^b=I_k^{b-1}=0,\;\; W_k^b - W_k^{b-1}\le \epsilon,
\\
\qquad\qquad\qquad\qquad\qquad\qquad\quad\text{intra-batch rejection},\\
W_k^b\le \omega,\qquad \qquad \qquad\quad\text{intra-batch early stopping},
\end{cases}
\end{equation}
where $I_k^b=0$ marks a winning round. Since $B_k$ remains a stopping time with respect to $\{I_k^1,\dots,I_k^{B_k}\}$, this rule preserves the e-process property.

\begin{algorithm*}[t]
\caption{Adaptive E-Process Sign-Flip Test (aeSFT)}
\label{alg:adaptive-syn}
\begin{algorithmic}[1]
\Require Trained models $f^{\mathrm{syn}},f$; level $\alpha\in(0,1)$; maximum test-data budget $M$; initial batch size $m_{\mathrm{init}}\in(0,M]$; growth factor $\beta>1$; intra-batch early stopping threshold $\omega\in(\alpha,1)$; rejection tolerance $\epsilon>0$; betting parameter $\hat\eta$ from~\eqref{eta-cal};
\State $k\gets 0$, \; $m_1\gets m_{\mathrm{init}}$, \; $\widetilde{W}_1^0\gets 1$.
\Loop
  \State $k \gets k+1$
  \State Draw fresh non-overlapping batch $\mathcal{D}_k^{\mathrm{test}}=\{Z_{k,i}\}_{i=1}^{m_k}$ \Comment{Fresh real test data}
  \State Compute $\delta_{k,i}$ and $\hat\Delta_k=(m_k)^{-1}\sum_i \delta_{k,i}$; set $W_{k}^{0}\gets 1$, $b\gets 0$ \Comment{Initialize batch}
  \Loop \Comment{Within-batch sign-flip test}
     \State $b\gets b+1$
     \State Draw signs $(s_{k,i}^{b})$, compute $\hat\Delta_{k}^{b}$ and $I_k^b$
     \State Update $W_k^{b}\gets W_{k}^{b-1}\,G_k^b(I_k^b)$
     \State $\widetilde{W}_{k}^{b}\gets W_k^b \times  \prod_{j=1}^{k-1} W_{j}^{B_j}$
      \If{$\widetilde{W}_{k}^{b}\ge 1/\alpha$} \Return \textsc{Reject $H_0^{\mathrm{syn}}$}  \Comment{Inter-batch rejection}\EndIf
      \If{$b\ge 2$ and $I_k^b=I_k^{b-1}=0$ and $ W_k^b-W_k^{b-1}\le\epsilon$} $B_k=b$, \textbf{break} \Comment{Intra-batch rejection}\EndIf
     \If{$W_k^b\le \omega$} $B_k=b$, \textbf{break} \Comment{Intra-batch early stopping}\EndIf
  \EndLoop
  \If{$W_k^{B_k}\le \omega$} \Comment{Intra-batch early stopping: enlarge next batch}
       \State $m_{k+1}\gets\lceil\beta\,m_k\rceil$ %\Comment{Enlarge next batch}
  \Else \Comment{Intra-batch rejection: keep batch size}
       \State $m_{k+1}\gets m_k$
  \EndIf
    \If{$\sum_{j=1}^{k+1}m_j>M$} \Return \textsc{Fail to reject $H_0^{\mathrm{syn}}$} \EndIf
\EndLoop
% \State \textbf{Output:} decision and anytime-valid \emph{p}-value $p=1/\max_{(k',b')\le(k,b)} \widetilde{W}_{k'}^{b'}$.
\end{algorithmic}
\end{algorithm*}

\paragraph{Test dataset budget}
While the inter-batch e-process already makes the overall data consumption $\sum_k m_k$ adaptive, further adapting the size $m_k$ of each individual batch improves the power. To economize on test data, one would like to begin testing with a small initial batch size $m_1=m_{\mathrm{init}}$. A small batch, however, may be insufficient to represent the ground-truth distribution $P$, so that sampling randomness masks a real but moderate improvement and the corresponding batch ends in intra-batch early stopping. To address this without sacrificing the small initial size, we enlarge the next batch immediately whenever the current batch ends in intra-batch early stopping. The next size $m_{k+1}$ is chosen \emph{before} drawing batch $k+1$, using only information from batches $1,\dots,k$. Thus, $m_{k+1}$ is measurable with respect to the history through the end of batch $k$, which is the predictability condition used in Theorem~\ref{prop:eprocess}:
\begin{equation}
m_{k+1}=
\begin{cases}
\lceil\beta\,m_k\rceil, & \text{if batch $k$ ended in early stopping},\\[2pt]
m_k, & \text{otherwise},
\end{cases}
\end{equation}
with growth factor $\beta>1$. The procedure declares \textsc{reject $H_0^{\mathrm{syn}}$} the first time $\widetilde{W}_{k}^{b}\ge 1/\alpha$, and otherwise keeps drawing batches while test data remain available. If the maximum budget $M$ is exhausted without rejection, the procedure concludes that $H_0^{\mathrm{syn}}$ cannot be rejected. Algorithm~\ref{alg:adaptive-syn} summarizes aeSFT.

% \paragraph{Remark on the within-batch bet.}
% In the standalone Fischer--Ramdas test, a ``bet everything on a win''
% short-circuit (effectively setting the betting probability to zero) is
% sometimes used to accelerate rejection within a batch. This short-circuit
% must be disabled here: it would force $W_k^{\mathrm{loc}}=0$ on any
% subsequent local loss, collapse $W_k^{\mathrm{agg}}$, and remove the
% recovery chance that the adaptive batch-size schedule is designed to
% provide.

\section{Experiments}\label{sec:experiments}

In this section, we empirically evaluate aeSFT via a binary classification task with synthetic training data, a wireless packet-scheduling task with synthetic environments, and a radio-map prediction task with low-fidelity ray-tracing data. The baselines include aMT~\cite{zecchin2025altt} described in Section~\ref{sec:background}, eSFT~\cite{fischer2025smc} presented in Section~\ref{Fischer_test}, and a fixed-sample $t$-test~\cite{hemerik2026multiple}. The $t$-test is based on the \emph{p}-value computed from the mean and variance of the paired differences $\delta$ at a prespecified sample size, and requires a further assumption that these differences follow a Gaussian distribution.

Unless otherwise stated, we set the target level to $\alpha=0.1$, use a maximum test-data budget $M=2{,}000$ for all methods, an initial batch size $m_{\mathrm{init}}=200$, an intra-batch early-stopping threshold $\omega=0.5$, a batch-growth factor $\beta=1.2$, and a change tolerance $\epsilon=0.1$ for aeSFT, and use the approximate growth-rate adaptive to the particular alternative (aGRAPA) betting strategy~\cite{zecchin2025altt} to determine the parameter $\mu_i$ in~\eqref{eq:amt-eprocess} for aMT.

\subsection{Binary Classification}
\label{sec:logistic-classify-experiment}

\noindent\textbf{Setting.} We consider a binary classification task with five-dimensional inputs and an $L_2$-regularized logistic-classification model trained using binary cross-entropy loss~\cite{simeone2022machine}. To construct the classification task, we generate a random orthogonal matrix $U$ by applying QR decomposition to a $5\times5$ matrix with i.i.d.\ standard Gaussian entries and set $\Sigma=U\Lambda U^\top$, where $\Lambda=\operatorname{diag}(1,\rho,\rho^2,\rho^3,\rho^4)$ with $\rho=0.65$. For each run, the entries of the weight vector $w$ and noise vector $\xi$ are independently drawn from uniform distributions on $[-1,1]$ and $[-0.25,0.25]$, respectively. Real and test inputs follow the distribution $P_X=\mathcal{N}(0,\Sigma)$, and their labels are generated as $y\mid x\sim\mathrm{Bernoulli}(\sigma(x^\top w))$. 

Synthetic inputs follow a different distribution $P_X^{\mathrm{syn}}=\mathcal{N}(0,\Sigma_r)$, where $\Sigma_r$ is the leading rank-$r$ approximation of $\Sigma$ obtained from its eigendecomposition, and their labels are similarly generated using a perturbed parameter $w^{\mathrm{syn}}=w+\xi$. Consequently, the synthetic data differ from the real data through both covariance matrix approximation and label-mapping noise. The dataset sizes are set to $n=100$ for the real dataset and $n^{\mathrm{syn}}=50$ for the synthetic dataset, respectively.

\noindent\textbf{Evaluation.} The ground-truth usefulness of a synthetic dataset $\mathcal{D}^{\mathrm{syn}}$ is evaluated via the true population loss difference $\Delta = \mathbb{L}(f)-\mathbb{L}(f^{\mathrm{syn}})$ between models $f$ and $f^{\mathrm{syn}}$, which is estimated on an independent held-out set of $10{,}000$ samples.

\noindent\textbf{Results.} First, we examine a setting in which the synthetic data are useful, i.e., the null~\eqref{eq:H_0} is false, and thus we have $\Delta>0$. Fig.~\ref{fig:logistic-classify-evidence-data} visualizes how the different methods accumulate evidence against the null as a function of consumed test data. eSFT uses a fixed test batch of size $M=2{,}000$, and thus its e-value remains at one until all $2{,}000$ data samples have been processed and then jumps at the decision point. Similarly, the $t$-test rejects the null after consuming all $2{,}000$ test samples, with its evidence value plotted as the reciprocal of its $p$-value.

\begin{figure}[t]
\centering
\includegraphics[width=0.5\textwidth]{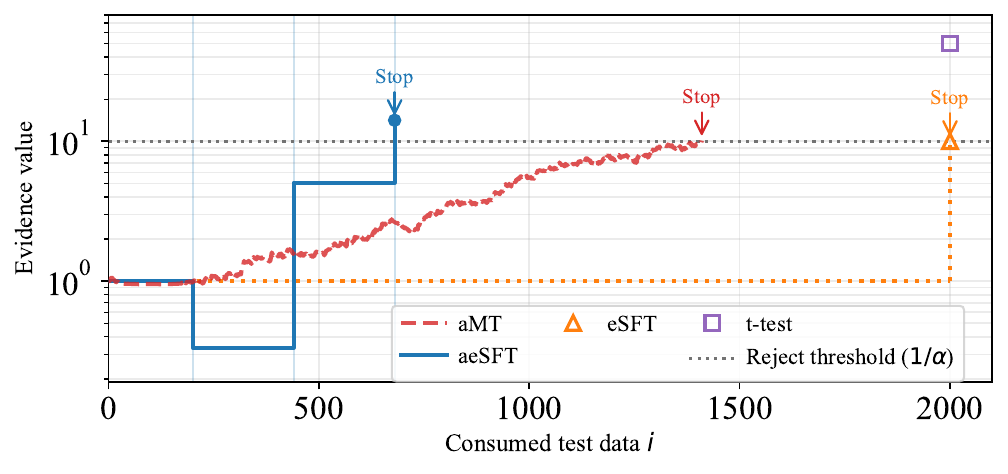}
\caption{\small Evidence accumulation versus consumed test data in a test instance where synthetic data are useful ($\Delta>0$). The reciprocal $t$-test \emph{p}-value is shown only as a visual score. The vertical axis is logarithmic.}
\label{fig:logistic-classify-evidence-data}
\end{figure}

In contrast, aeSFT processes test data in adaptive batches, producing a stepwise trajectory that reaches the rejection threshold after consuming $680$ data samples. The aMT baseline also updates its e-value after each individual test observation and reaches the rejection threshold after consuming approximately $1{,}410$ data samples. Thus, in this instance, aeSFT accumulates sufficient evidence for rejection using fewer test samples than all three baselines.

\begin{figure}[t]
\centering
\includegraphics[width=0.45\textwidth]{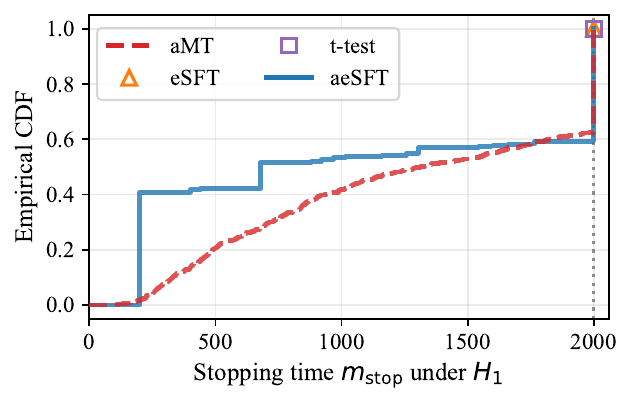}
\caption{\small Empirical CDF of stopping time $m_{\mathrm{stop}}$ under the alternative hypothesis $H_1$ (synthetic data are useful) for aMT, eSFT, $t$-test, and aeSFT in the binary classification task.}
\label{fig:logistic-classify-mstop-h1}
\end{figure}

Fig.~\ref{fig:logistic-classify-mstop-h1} plots the empirical cumulative distribution function (CDF) of the stopping time $m_{\mathrm{stop}}$ under the alternative hypothesis $H_1$, i.e., synthetic data are useful, for each method. The distribution is obtained by selecting runs with $\Delta>0$, i.e., runs in which synthetic data are useful. Both aMT and aeSFT can stop before the maximum budget on easier test instances, whereas eSFT and the $t$-test always make their decisions at the prespecified fixed sample size $M=2{,}000$. In particular, aeSFT places substantially more probability on small stopping times, demonstrating its ability to terminate early when the evidence is strong.

\begin{figure}[t]
\centering
\includegraphics[width=0.45\textwidth]{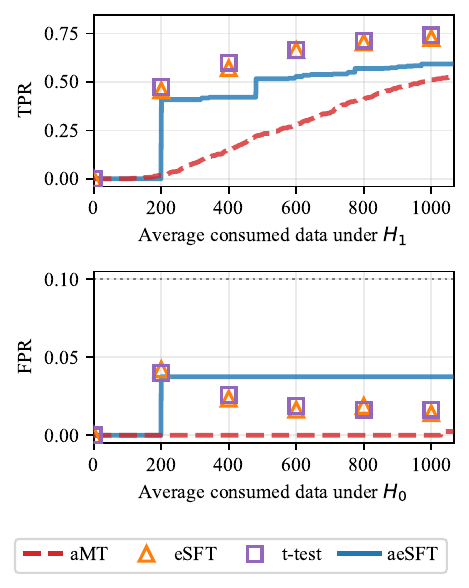}
\caption{\small TPR and FPR over $1{,}000$ independent synthetic data tests with target level $\alpha=0.1$ for aMT, eSFT, $t$-test, and aeSFT.}
\label{fig:logistic-classify-repetitions}
\end{figure}

To further evaluate the sample efficiency of aeSFT and aMT on average, we report the true-positive rate (TPR) and false-positive rate (FPR) as a function of the testing round $i$, corresponding to the use of $i$ data points. The TPR is given by
\begin{equation}
\operatorname{TPR}(i)=\mathbb{E}_{H_1}\!\left[\mathbf{1}\{D(Z^i)=\text{reject }H_0\}\right],
\end{equation}
and the FPR is given by
\begin{equation}
\operatorname{FPR}(i)=\mathbb{E}_{H_0}\!\left[\mathbf{1}\{D(Z^i)=\text{reject }H_0\}\right],
\end{equation}
where $D(Z^i)=D(Z^{m_{\mathrm{stop}}})$ for $i\geq m_{\mathrm{stop}}$. The corresponding average stopping time is evaluated as $\mathbb{E}_{H_j}[m_{\mathrm{stop}}(i)]$ for $j\in\{0,1\}$, where $m_{\mathrm{stop}}(i)=m_{\mathrm{stop}}$ for $i\geq m_{\mathrm{stop}}$ and $m_{\mathrm{stop}}(i)=i$ for $i<m_{\mathrm{stop}}$. The plotted FPR uses the same direct null $\Delta\leq0$ for every method. This empirical comparison differs from the formal guarantees: aMT controls Type-I error under the direct mean null $H_0$, whereas SFT, eSFT, and aeSFT control it under the symmetry null $H_0^{\mathrm{syn}}$. Among the $1{,}000$ independent test instances, $573$ satisfy $H_1$ and $427$ satisfy $H_0$. Fig.~\ref{fig:logistic-classify-repetitions} plots TPR and FPR against average consumed data for aMT, eSFT, $t$-test, and aeSFT.

For reference, each eSFT and $t$-test point corresponds to a separately prespecified fixed dataset size. For example, if $M=200$ is selected before testing, one cannot inspect the result and then increase $M$ to continue the same test, which would result in $p$-hacking and may compromise the Type-I error guarantee. In contrast, aMT and aeSFT permit statistically valid continuation, thereby offering greater flexibility to improve power when the initial batch is insufficient. Among the adaptive schemes, as shown in Fig.~\ref{fig:logistic-classify-repetitions}, aeSFT achieves significantly higher TPR than aMT, especially when the average consumed test dataset size is small, demonstrating a higher sample efficiency. Because aeSFT begins with an initial batch size of $m_{\mathrm{init}}=200$ points, its TPR curve is necessarily stepwise and cannot rise before the initial batch has been consumed. Its empirical FPR under the null $H_0$ remains below the target level $\alpha=0.1$.

\begin{figure}[t]
\centering
\includegraphics[width=0.5\textwidth]{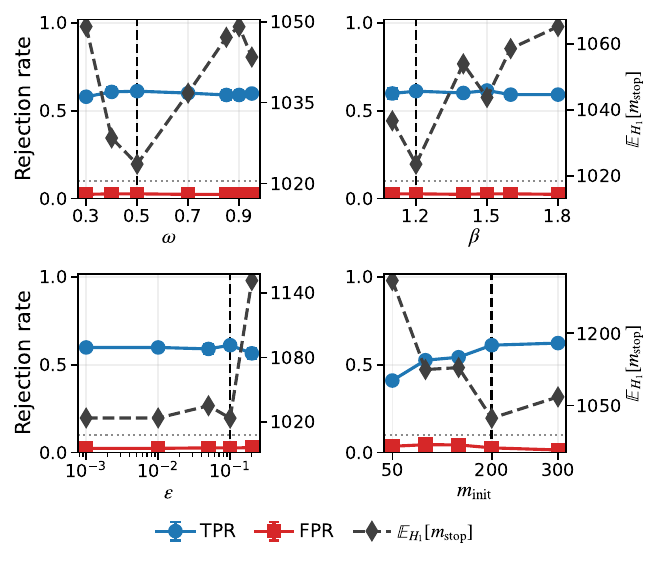}
\caption{\small aeSFT parameter ablation in the binary classification experiment. Solid curves show TPR and FPR, diamonds show $\mathbb{E}_{H_1}[m_{\mathrm{stop}}]$, black dashed lines mark the selected values, and horizontal dotted lines mark $\alpha=0.1$.}
\label{fig:aesft-ablation}
\end{figure}

\begin{figure*}[t]
\centering
\includegraphics[width=0.98\textwidth]{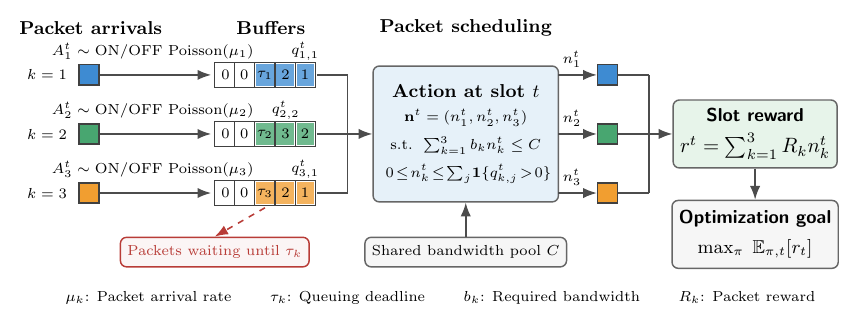}
\caption{\small Three-class packet-scheduling task~\cite{vanhuynh2019optimal}. The scheduling policy allocates a shared bandwidth pool for packets to maximize the expected network reward.}
\label{fig:network-packet-task}
\end{figure*}

Fig.~\ref{fig:aesft-ablation} reports a repeated one-at-a-time ablation of the four aeSFT design parameters. We treat a TPR as near-optimal when its difference from the largest mean TPR is within the combined $98\%$ standard-error scale, and then select the parameter setting with the smallest $\mathbb{E}_{H_1}[m_{\mathrm{stop}}]$. The resulting setting is $\omega=0.5$, $\beta=1.2$, $\epsilon=0.1$, and $m_{\mathrm{init}}=200$, for which the mean TPR is $0.612\pm0.026$, the mean FPR is $0.027\pm0.007$, and $\mathbb{E}_{H_1}[m_{\mathrm{stop}}]=1{,}024$ samples. For comparison, $m_{\mathrm{init}}=300$ slightly raises the mean TPR to $0.624$ but also raises $\mathbb{E}_{H_1}[m_{\mathrm{stop}}]$ to $1{,}068$ samples. We therefore retain $m_{\mathrm{init}}=200$ as the more data-efficient choice.

\subsection{Wireless Network Packet Scheduling}
\label{sec:network-packet-experiment}

\noindent\textbf{Setting.} We now consider a packet-scheduling task in wireless networks~\cite{vanhuynh2019optimal,tao2025digital}. As illustrated in Fig.~\ref{fig:network-packet-task}, data packets belong to three classes with class-dependent arrival rates $\mu_k$ for $k=1,2,3$. A Pareto ON/OFF-modulated Poisson process is employed to model traffic bursts and temporal variability~\cite{8718631}. For each class, the packet-count rate switches between $3\mu_k$ in the ON state, i.e., burst periods, and $0.25\mu_k$ in the OFF state, i.e., idle periods. The ON and OFF durations follow Pareto distributions with shape parameters $1.4$ and $1.6$ and scale parameters $3$ and $10.5$ slots, respectively. These scales yield a long-run ON fraction of approximately $0.25$, preserving the mean arrival rate $\mu_k$. Packets of class $k$ require $b_k$ units of bandwidth when served and yield network reward $R_k$ each. Packets not served immediately can remain in the corresponding buffer for at most $\tau_k$ time slots.

Let $q_{k,j}^t\in[0,\tau_k]$ denote the remaining waiting time of the $j$-th packet in the class-$k$ buffer at slot $t$. Occupied positions are ordered by increasing remaining time and followed by zero padding up to the buffer capacity $N_{\max}$. The state is given by
$
s^t={(q_{k,j}^t)}_{k=1,\ldots,3;j=1,\ldots,N_{\max}}.
$
Each unserved packet's remaining time decreases by one after a slot, and the packet expires upon reaching zero. Based on this state, a packet scheduler selects how many packets of each class to serve at each time slot. The resulting decision, denoted by
$
\bm n^t=(n_1^t,n_2^t,n_3^t),
$
is subject to $0\leq n_k^t\leq \sum_{j}\mathbf{1}\{q_{k,j}^t>0\}$ as well as the shared-bandwidth constraint
$
\sum_{k=1}^{3}b_k n_k^t\leq C.
$
The corresponding total reward at time slot $t$ is
$
r^t=\sum_{k=1}^{3}R_k n_k^t.
$
The optimization goal is to learn a policy that maximizes the long-term average reward in the network, thereby balancing immediate packet rewards against congestion and deadline expiration.

To generate synthetic training data for this task, we construct a DT environment of the considered system. In particular, we perturb the real arrival rate and deadline of each packet class as
\begin{equation}
\mu_k^{\mathrm{syn}}=\mu_k+\xi_{k}^\mu,
\qquad
\tau_k^{\mathrm{syn}}=\tau_k+\xi_{k}^\tau,
\end{equation}
where the perturbations $\xi_{k}^\mu$ and $\xi_{k}^\tau$ are drawn from uniform distributions on $[-\mu_k/2, \mu_k/2]$ and $[-\tau_k/2, \tau_k/2]$, respectively, across repeated trials. The bandwidth requirements, packet rewards, shared capacity, state definition, and action space are kept unchanged.

We collect a fixed real dataset $\mathcal{D}$ consisting of trajectories generated in the real environment and a synthetic dataset $\mathcal{D}^{\mathrm{syn}}$ consisting of trajectories generated in the DT environment. A trajectory is a sequence of states, actions, rewards, and state transitions generated during one decision episode. We use deep Q-learning as the learning algorithm $A(\cdot)$. Training is offline: the Q-network learns only from the trajectories stored in these datasets, without further interaction with either environment~\cite{prudencio2024offline}. Both models $f$ and $f^{\mathrm{syn}}$ are implemented as deep Q-networks (DQNs) with the same architecture~\cite{mnih2015human}, and are trained on $\mathcal{D}$ and $\mathcal{D}\cup\mathcal{D}^{\mathrm{syn}}$, respectively. The complete configuration is summarized in Appendix~\ref{app:experiment-config}.

\noindent\textbf{Evaluation.} For a real-environment interaction episode $Z$ of $T=60$ slots, the implemented evaluation loss is the complement of the average normalized episodic return,
\begin{equation}
\mathbb{L}(f;Z)=1-\frac{1}{\bar{r}\cdot T}\sum_{t=1}^{T}r^t\in[0,1],
\label{eq:wireless-episode-loss}
\end{equation}
where $r^t$ is the unnormalized slot reward attained by model $f$ and $\bar{r}$ is a valid per-slot reward upper bound. Consequently, the paired loss difference $\delta=\mathbb{L}(f;Z)-\mathbb{L}(f^{\mathrm{syn}};Z)$ equals the difference of normalized returns of models $f$ and $f^{\mathrm{syn}}$. The loss is bounded with $L_{\max}=1$, so aMT is applicable.

After training the two models, we determine whether the synthetic data from the DT are actually useful in real deployment. As in the binary classification task of Section~\ref{sec:logistic-classify-experiment}, we define ground-truth usefulness by evaluating the direct population gap $\Delta$ between $f$ and $f^{\mathrm{syn}}$ over $200$ held-out real-environment interaction episodes. Among the $200$ test instances, $99$ satisfy $H_1$ and $101$ satisfy $H_0$.

\noindent\textbf{Results.} Fig.~\ref{fig:network-packet-mstop-h1} plots the empirical CDF of the stopping time $m_{\mathrm{stop}}$ under the alternative hypothesis $H_1$ (synthetic data are useful) for each method in the packet-scheduling task. The figure shows that both adaptive schemes can terminate before exhausting all the episodes on easier test instances. aeSFT allows much smaller stopping times than aMT, while eSFT and the $t$-test make decisions only at their prespecified fixed sample size.

\begin{figure}[t]
\centering
\includegraphics[width=0.45\textwidth]{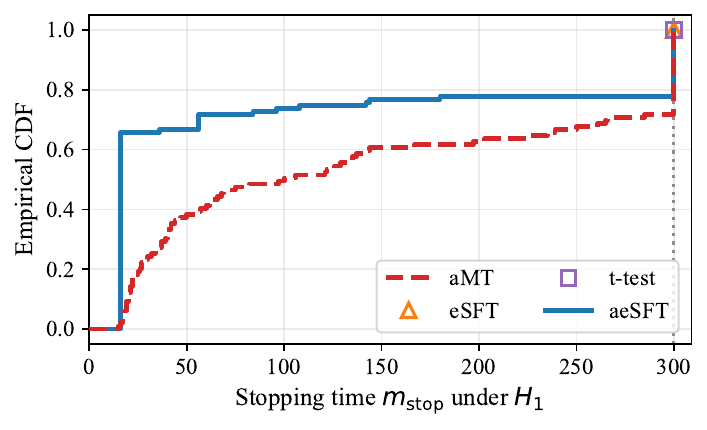}
\caption{\small Empirical CDF of stopping time $m_{\mathrm{stop}}$ under the alternative hypothesis $H_1$ (synthetic data are useful) for aMT, eSFT, $t$-test, and aeSFT in the wireless packet scheduling task.}
\label{fig:network-packet-mstop-h1}
\end{figure}

As in Fig.~\ref{fig:logistic-classify-repetitions}, Fig.~\ref{fig:network-packet-results} plots TPR and FPR as a function of the test episodes consumed by aMT, eSFT, $t$-test, and aeSFT. The figure shows that aeSFT detects useful synthetic trajectories substantially earlier than aMT. After consuming the initial batch of $m_{\mathrm{init}}=16$ real test episodes, aeSFT attains a TPR of $0.657$, whereas aMT remains near zero. The TPR of aeSFT reaches $0.717$ at $31.8$ episodes and $0.778$ at the final average consumption of $88.0$ episodes. In contrast, aMT remains at zero through $15$ episodes and reaches TPRs of $0.505$ and $0.717$ only after average consumptions of $68.9$ and $143.6$ episodes, respectively. For reference, eSFT and the $t$-test achieve TPR levels ranging from approximately $0.7$ to $0.85$, but at the cost of having to prespecify a fixed sample size for each test.

\begin{figure}[t]
\centering
\includegraphics[width=0.45\textwidth]{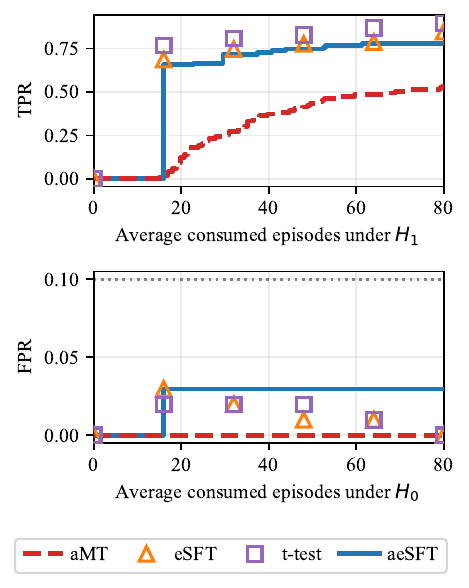}
\caption{\small TPR and FPR versus average consumption of real-network test episodes for the packet-scheduling task, with target level $\alpha=0.1$ for aMT, eSFT, $t$-test, and aeSFT.}
\label{fig:network-packet-results}
\end{figure}

The lower panel shows that all methods remain well below the target FPR of $\alpha=0.1$. Thus, in this network-control task, aeSFT provides a favorable compromise: it maintains a low FPR, identifies useful synthetic environments with markedly fewer real-network evaluation episodes, and retains the flexibility to continue testing with fresh data.

\subsection{Radio-Map Prediction}
\label{sec:radiomap-experiment}

\noindent\textbf{Setting.} Radio maps describe the signal strength of a base station across
a geographical area and are important inputs to wireless-network planning.
Physics-based ray tracing, such as NVIDIA Sionna RT~\cite{Hoydis2023SionnaRT},
can generate accurate radio maps from a three-dimensional (3D) scene, but its cost
grows rapidly when the maximum reflection depth and the number of rays per transmitter (Tx) are increased. For example, a high-accuracy setup may
use a reflection depth, i.e., the maximum number of times a radio wave ray can bounce off objects, of 20 and $1\times10^9$ rays/Tx~\cite{11575744}. This motivates neural-network-based predictors
such as PMNet~\cite{Lee2023PMNet} and AIRMap~\cite{11575744}, which learn to predict a complete radio map
directly from geometric maps, without the need to run ray tracing at inference time.

\begin{figure}[t]
\centering
\includegraphics[width=0.5\textwidth]{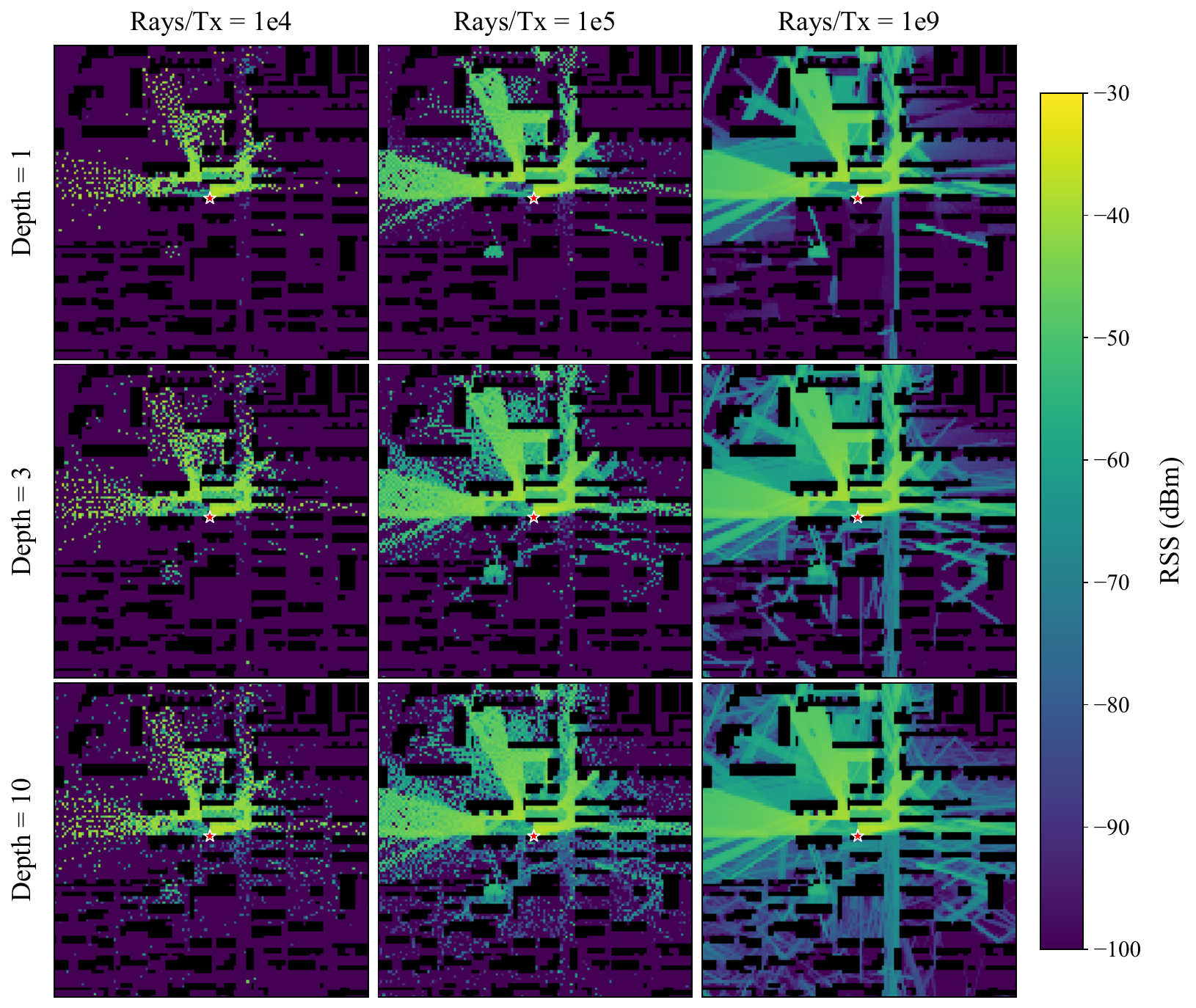}
\caption{\small Sionna RT~\cite{Hoydis2023SionnaRT} radio maps for a fixed $512\,\mathrm{m}\times512\,\mathrm{m}$ scene imported from OpenStreetMap~\cite{haklay2008openstreetmap} and a base-station location (star marker) under different reflection depths, i.e., maximum numbers of reflections, and rays/Tx settings.}
\label{fig:radiomap-fidelity}
\end{figure}

To elaborate, adopting a setting similar to \cite{11575744} and \cite{Lee2023PMNet}, we consider a $512\,\mathrm{m}\times512\,\mathrm{m}$ area discretized into a
$128\times128$ grid. The input of the radio map prediction consists of the building-height map together with the base-station's location and height, and the output of the prediction is the
received signal strength (RSS) at every grid cell. The real-world geometric scenes are imported from OpenStreetMap~\cite{haklay2008openstreetmap}. Fig.~\ref{fig:radiomap-fidelity} illustrates the effect of ray-tracing fidelity for a given scene, with black areas representing buildings and a given base-station location marked with a star. Small ray budgets produce sparse maps with large uncovered
regions, whereas increasing the number of rays and the reflection depth resolves more
non-line-of-sight paths and yields a denser radio map. These improvements, however, come at the cost of increased computational complexity.

We treat maps generated at depth 20 with $10^9$ rays/Tx as the
high-fidelity real data $\mathcal{D}$, while low-fidelity synthetic maps $\mathcal{D}^{\mathrm{syn}}$ are generated on
a $14\times14$ grid comprising depths ranging from $1$ to $14$ and rays/Tx ranging from $10^4$ to $10^5$ logarithmically spaced. This yields $196$ low-fidelity synthetic datasets $\mathcal{D}^{\mathrm{syn}}$ with different ray-tracing fidelities. The 3D scene, Tx layout, and ray seed are shared across all 196 settings, isolating the effect of synthetic-data fidelity. 

We train one baseline model $f$ based on dataset $\mathcal{D}$ and one synthetic-data-augmented model $f^{\mathrm{syn}}$ for every ray-tracing fidelity. The real dataset comprises $64$ high-fidelity samples, and the synthetic datasets comprise $192$ low-fidelity maps each. The loss is the masked per-map mean squared error (MSE) after normalizing RSS to $[0,1]$~\cite{Lee2023PMNet}. All models are trained using the same architecture and hyperparameters, with the complete configuration summarized in Appendix~\ref{app:experiment-config}. We use a maximum budget of $400$ high-fidelity test maps, an initial batch size of $m_{\mathrm{init}}=50$, and a batch-growth factor $\beta=1.6$.

\noindent\textbf{Evaluation.} For each synthetic model, the paired improvement on a held-out high-fidelity map $Z$ is given by $\delta=\mathbb{L}(f;Z)-\mathbb{L}(f^{\mathrm{syn}};Z)$. We determine its ground-truth usefulness using the average improvement over 1,000 independent high-fidelity maps. This yields 123 useful settings under $H_1$ and 73 non-useful settings under $H_0$. All statistical tests are then run on a separate set of 400 high-fidelity maps. For aMT, we set $L_{\max}$ to the maximum loss over all test samples across all models.

% \begin{figure}[t]
% \centering
% \includegraphics[width=\textwidth]{ExpResult/radiomap_usefulness_grid.pdf}
% \caption{\small Ground truth usefulness and test decisions over the 196 low-fidelity synthetic datasets, compared at approximately 150 average consumed high-fidelity maps. A black cell denotes a useful dataset in the oracle panel and rejection of the corresponding null in each method panel.}
% \label{fig:radiomap-usefulness-grid}
% \end{figure}

\begin{figure}[t]
\centering
\includegraphics[width=0.45\textwidth]{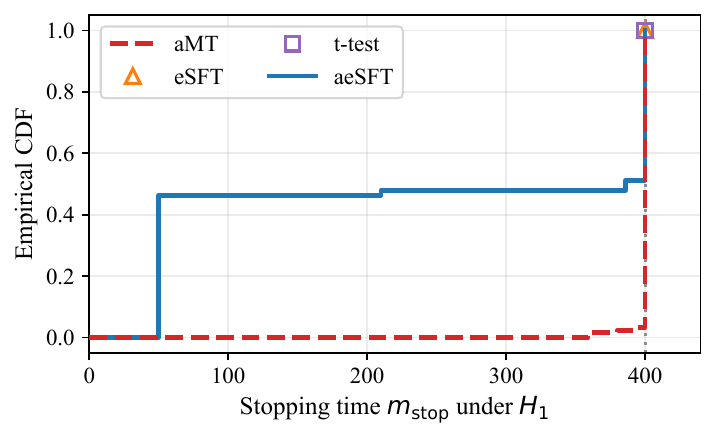}
\caption{\small Empirical CDF of stopping time $m_{\mathrm{stop}}$ under the alternative hypothesis $H_1$ (synthetic data are useful) for aMT, eSFT, $t$-test, and aeSFT in the radio-map prediction task.}
\label{fig:radiomap-mstop-h1}
\end{figure}

\noindent\textbf{Results.} Fig.~\ref{fig:radiomap-mstop-h1} reports the empirical CDF of the stopping time
for useful synthetic datasets, i.e., under the alternative hypothesis $H_1$. In this experiment, aMT makes no rejection throughout the first $350$ maps and identifies only $4$ of the $123$ useful synthetic datasets by the full $400$-map budget. In contrast, aeSFT identifies $57$ of the $123$ useful synthetic datasets after consuming the initial batch of $50$ maps. eSFT and
the $t$-test make their decisions only at their separately prespecified fixed
sample sizes.

\begin{figure}[t]
\centering
\includegraphics[width=0.45\textwidth]{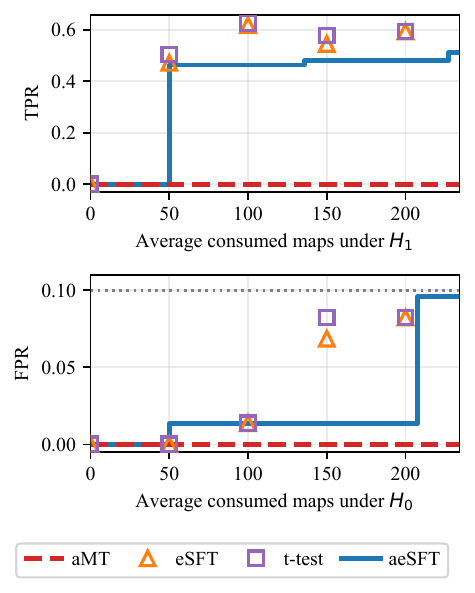}
\caption{\small TPR and FPR over the 196 settings versus the average number of consumed high-fidelity radio maps, with target level $\alpha=0.1$ for aMT, eSFT, $t$-test, and aeSFT.}
\label{fig:radiomap-test-results}
\end{figure}

As in the preceding experiments, Fig.~\ref{fig:radiomap-test-results} compares TPR and FPR as functions of the average number of consumed high-fidelity test maps. After the initial batch of $m_{\mathrm{init}}=50$ maps, aeSFT already attains a TPR of $0.463$, whereas aMT remains at zero. With adaptive continuation, aeSFT increases its TPR to $0.480$ at an average consumption of $135.9$ maps and to $0.512$ at $227.4$ maps under the alternative hypothesis $H_1$. Its empirical FPR is $0.014$ after the initial batch and rises to $0.096$ at an average consumption of $207.8$ maps under $H_0$, remaining below $\alpha=0.1$. In contrast, aMT makes no rejection over the displayed consumption range.

The gain of aeSFT in early detection observed in Fig.~\ref{fig:radiomap-test-results} is obtained while maintaining controlled empirical false-alarm behavior. Specifically, the FPR of aeSFT is 0.014 after the first batch, remains at 0.014 through an average consumption of approximately 200 maps under $H_0$, and increases to 0.096 with further continuation, remaining below the target level
$\alpha=0.1$. For reference, eSFT and the $t$-test reach higher TPRs at some fixed sample sizes, but each point requires a separately prespecified test size and does not permit continuation of the same test.

\begin{figure}[t]
\centering
\includegraphics[width=0.45\textwidth]{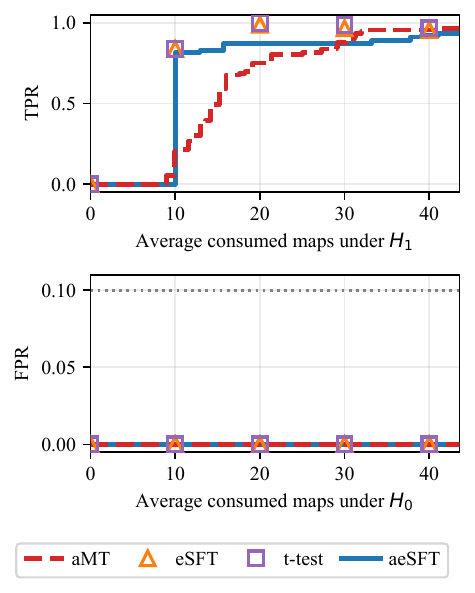}
\caption{\small TPR and FPR for the IoU metric over the 196 radio-map settings versus the average number of consumed high-fidelity radio maps, with target level $\alpha=0.1$ for aMT, eSFT, $t$-test, and aeSFT.}
\label{fig:radiomap-iou-results}
\end{figure}

While MSE is widely used to evaluate radio-map prediction accuracy, downstream wireless communication applications such as user scheduling~\cite{7560605} and base-station coverage optimization~\cite{tao2026intelligentbasestationdeployment} are often more concerned with the spatial extent over which sufficiently strong signal can be provided. To capture this information, we define the promising region of a radio map as the set of pixels whose RSS is at least $-80\,\mathrm{dBm}$, and evaluate its prediction using the intersection over union (IoU) between the predicted and ground-truth promising regions. Accordingly, the per-map loss is defined as $\mathbb{L}_{\mathrm{IoU}}=1-\mathrm{IoU}\in[0,1]$, since a larger IoU indicates better prediction. This loss has the deterministic bound $L_{\max}=1$, directly satisfying the bounded-loss requirement of aMT. We recompute the ground-truth usefulness of each synthetic dataset using the mean IoU improvement over the same $1{,}000$ held-out maps, yielding $93$ settings under $H_1$ and $103$ under $H_0$.

Fig.~\ref{fig:radiomap-iou-results} shows that aeSFT retains the benefit of adaptive data consumption while achieving detection performance comparable to the fixed-sample baselines. Specifically, aeSFT attains a TPR of $0.925$ after consuming $40$ maps on average under $H_1$, while eSFT and the $t$-test attain TPRs of $0.957$ and $0.968$, respectively, at the same prespecified sample size of $40$. We also observe that aMT attains a TPR of $0.957$ after consuming $40$ maps on average, representing a substantial improvement over its performance in the MSE-based experiment. This observation is consistent with the earlier discussion of the direct test: the IoU loss is naturally bounded and exhibits less pronounced heavy-tail behavior in loss differences, which provide more favorable conditions for mean-based sequential testing. All four methods maintain zero empirical FPR.

\section{Conclusion}
\label{sec:conclusion}

DTs and learned world models can alleviate the scarcity of real training data, but the sim-to-real gap means that synthetic augmentation does not reliably improve real-world performance. This work addressed how to determine, given a real training dataset, a candidate synthetic dataset, and a fixed learning algorithm, whether augmentation improves population-level performance while using as few real test data points as possible. Specifically, we introduced aeSFT to adapt both the number of Monte Carlo sign-flip rounds and the amount of real test data consumed. Its intra-batch e-process measures evidence within the current batch, while its inter-batch e-process compounds evidence across independent batches. Under the symmetry null, stopping when the accumulated evidence first exceeds $1/\alpha$ controls the Type-I error at level $\alpha$ without requiring a prespecified test-set size or a known finite loss bound.

Across binary classification, wireless packet scheduling, and radio-map prediction, aeSFT consistently detected useful synthetic data earlier than aMT. In the binary classification task, aeSFT attained a TPR above $0.5$ after consuming $500$ test samples, whereas aMT required over $1{,}000$ samples to reach the same TPR on average. In packet scheduling, aeSFT attained TPR $0.657$ after its initial $16$ real test episodes, compared with near zero for aMT, and its mean stopping time under $H_1$ was $88.0$ episodes versus $143.6$ for aMT. In radio-map prediction, aeSFT attained a TPR of $0.463$ after $50$ high-fidelity maps while aMT remained at zero; adaptive continuation raised its TPR to $0.512$ while keeping the empirical FPR below $\alpha=0.1$. Fixed-sample eSFT and the paired $t$-test achieved comparable power at some prespecified sample sizes, but they do not allow the same test to continue with additional data. Overall, aeSFT provides early detection and anytime-valid continuation while maintaining the FPR below the target level.

In general, aeSFT is applicable to any setting in which the paired loss difference can be evaluated on real test data, especially when trusted real evaluations are expensive. In wireless communications, it can evaluate synthetic channel states, radio maps, or network-control trajectories generated by network DTs using limited field measurements or high-fidelity simulations. For large language models, it can assess synthetic instruction, reasoning, or preference datasets using a limited set of expert-annotated prompts or human-preference evaluations. In robotics and autonomous driving, it can validate world-model-generated perception data and control trajectories using real-world rollouts. By allocating additional real evaluations only to candidate synthetic datasets that remain unresolved, aeSFT provides a statistically controlled and efficient validation layer between synthetic-data generation and its downstream use.

\begin{appendices}

\section{Proof of Theorem~\ref{prop:eprocess}}\label{app:proof}

We prove that, under the null $H_0^{\mathrm{syn}}$ in \eqref{eq:H0-prime} and the conditions stated in Theorem~\ref{prop:eprocess}, the inter-batch wealth growth process $\widetilde{W}_{k}^{b}$ defined in \eqref{eq:inter-batch-e-process} is an e-process, so that the rejection rule $\widetilde{W}_{k}^{b}\ge 1/\alpha$ controls the Type-I error at level $\alpha$.

\paragraph{Intra-batch e-process}
We first establish that the intra-batch wealth growth process in each batch is an e-process. Fix a batch $k$ and let $\mathcal{F}_k^{b}$ denote the intra-batch test history, which is defined as the $\sigma$-algebra generated by the sign-flip indicators observed up to step $b$, i.e.,
\begin{equation}
     \mathcal{F}_k^{b}=\sigma(I_k^1,\dots,I_k^b).
\end{equation}
By the budget constraint \eqref{eq:budget-constraint}, under $H_0^{\mathrm{syn}}$ the betting function satisfies $\mathbb{E}_{H_0^{\mathrm{syn}}}[G_k^b(I_k^b)\mid\mathcal{F}_k^{b-1}]= 1$, so the intra-batch process $(W_k^b)_{b\ge 0}$, started at $W_k^0=1$, is a nonnegative test martingale~\cite{fischer2025smc}, i.e.,
\begin{equation}
\mathbb{E}_{H_0^{\mathrm{syn}}}\bigl[W_k^b\mid\mathcal{F}_k^{b-1}\bigr]= W_k^{b-1}.
\label{eq:intra-martingale}
\end{equation}
This is the fair-game property within casino $k$: under the null, a single bet cannot increase the gambler's wealth in expectation. The number of rounds $B_k$ is a stopping time with respect to $(\mathcal{F}_k^b)_{b\ge0}$, that is, the decision of whether to stop at step $b$ uses only the batch data and the indicators observed up to step $b$. By the optional stopping theorem for nonnegative test martingales~\cite{ramdas2025evalues}, the stopped wealth therefore satisfies
\begin{equation}
\mathbb{E}_{H_0^{\mathrm{syn}}}\bigl[W_k^{B_k}\bigr]\le 1,
\label{eq:intra-leq-1}
\end{equation}
which identifies each stopped intra-batch wealth as a valid e-value under $H_0^{\mathrm{syn}}$.

\paragraph{Inter-batch e-process}
We then show that the inter-batch product is an e-process. The stopping time of the entire procedure depends on the accumulated wealth, which in turn depends on the history of all sign-flip indicators and test batches observed across batches $1,\dots,k$, so we consider the joint history of all batches together.

We index the steps of the whole procedure by the pairs $(k,b)$ ordered lexicographically, where step $(k,b)$ is the $b$-th sign-flip round of batch $k$, and let $\mathcal{H}_k^{b}$ denote the inter-batch testing history, defined as the $\sigma$-algebra generated by sign-flip indicators observed up to step $(k,b)$, i.e.,
\begin{equation}
\mathcal{H}_k^b=\sigma\bigg(
\{I_k^1,\dots,I_k^b\}
\cup\bigcup_{j=1}^{k-1}\{I_j^1,\dots,I_j^{B_j}\}\bigg).
\end{equation}

Within batch $k$, the inherited wealth $\prod_{j=1}^{k-1}W_j^{B_j}$ is $\mathcal{H}_k^{b-1}$-measurable. The batch size $m_k$ is selected from the past before batch $k$ is drawn, and the fresh observations in that batch are i.i.d.\ and independent of all previous batches. Therefore, conditioning on the cross-batch history does not change the conditional law of the current batch, and \eqref{eq:intra-martingale} continues to hold after conditioning on $\mathcal{H}_k^{b-1}$. The inter-batch process is consequently a nonnegative test martingale by \eqref{eq:inter-batch-e-process}:
\begin{align}
\mathbb{E}_{H_0^{\mathrm{syn}}}\bigl[\widetilde{W}_k^{b}\mid\mathcal{H}_k^{b-1}\bigr]
&=\Bigl(\textstyle\prod_{j=1}^{k-1}W_j^{B_j}\Bigr)\,
\mathbb{E}_{H_0^{\mathrm{syn}}}\bigl[W_k^{b}\mid\mathcal{H}_k^{b-1}\bigr]\notag\\
&=\Bigl(\textstyle\prod_{j=1}^{k-1}W_j^{B_j}\Bigr)\,
\mathbb{E}_{H_0^{\mathrm{syn}}}\bigl[W_k^{b}\mid\mathcal{F}_k^{b-1}\bigr]\notag\\
&=\Bigl(\textstyle\prod_{j=1}^{k-1}W_j^{B_j}\Bigr) W_k^{b-1}=\widetilde{W}_k^{b-1}.
\label{eq:within-batch}
\end{align}

By the optional stopping theorem for nonnegative test martingales~\cite{ramdas2025evalues}, for any stopping time $(K,B_K)$ of the whole procedure with respect to $\mathcal{H}_k^b$, the stopped wealth satisfies
\begin{equation}
\mathbb{E}_{H_0^{\mathrm{syn}}}\bigl[\widetilde{W}_K^{B_K}\bigr]\le 1,
\label{eq:inter-leq-1}
\end{equation}
which establishes that $(\widetilde{W}_k^b)$ is an e-process.

By Ville's inequality, this in turn gives
\begin{equation}
\Pr\nolimits_{H_0^{\mathrm{syn}}}\Bigl(\exists(k,b):\;\widetilde{W}_{k}^{b}\ge 1/\alpha\Bigr)
\le \alpha,
\end{equation}
which establishes Type-I control of aeSFT at level $\alpha$.
\hfill$\square$

\section{Experimental Configurations}\label{app:experiment-config}

The configurations of the wireless packet-scheduling and radio-map prediction experiments are summarized in Tables~\ref{tab:wireless-config} and~\ref{tab:radiomap-config}, respectively.

\begin{table}[t]
\centering
\caption{Configuration of the wireless packet-scheduling experiment.}
\label{tab:wireless-config}
\small
\begin{tabular}{>{\raggedright\arraybackslash}p{0.18\linewidth}p{0.73\linewidth}}
\toprule
Component & Setting \\
\midrule
Environment & Arrival rates $\bm\mu=(4.0,1.8,0.8)$; deadlines $\bm\tau=(6,3,1)$; bandwidths $\bm b=(1,2,4)$; rewards $\bm R=(1,4,10)$; shared capacity $C=10$. \\
State and action & Three deadline-bucket buffers with $D_{\max}=8$ and per-bucket cap $N_{\max}=30$, giving a $24$-dimensional normalized state; $66$ integer bandwidth allocations satisfying $\sum_k u_k=C$. \\
Episode and dataset & $T=60$ slots; $|\mathcal{D}|=10$ real trajectories ($600$ transitions); $|\mathcal{D}^{\mathrm{syn}}|=30$ synthetic trajectories ($1{,}800$ transitions). \\
Behavior data policy & Greedy immediate-reward action with probability $0.55$ and a uniformly selected feasible action with probability $0.45$. \\
DQN architecture & State and three normalized action counts as input; two fully connected ReLU hidden layers of width $64$; scalar Q-value output. \\
DQN training & Adam with learning rate $10^{-3}$, Smooth-$L_1$ loss, discount $0.97$, minibatch size $128$, target update every $25$ gradient steps, and gradient-norm clipping at $5$. The real-only model uses $120$ updates and the pooled-data model uses $480$ updates, preserving the number of passes over the data. \\
\bottomrule
\end{tabular}
\end{table}

\begin{table}[t]
\centering
\caption{Configuration of the radio-map prediction experiment.}
\label{tab:radiomap-config}
\footnotesize
\begin{tabular}{>{\raggedright\arraybackslash}p{0.16\linewidth}p{0.76\linewidth}}
\toprule
Component & Setting \\
\midrule
Spatial grid & A $512\,\mathrm{m}\times512\,\mathrm{m}$ window with $4\,\mathrm{m}$ cells, giving a two-channel $128\times128$ input and a one-channel $128\times128$ RSS prediction. \\
Input encoding & Channel 1 is the 3D geometry represented as a building-height grid and normalized to $[0,1]$. Channel 2 is zero everywhere except at the grid cell containing the base station, whose positive value is the base-station height normalized to $[0,1]$. \\
Transmitter setting & Carrier frequency $3.5\,\mathrm{GHz}$; transmit power $40\,\mathrm{dBm}$; $1\times1$ planar transmit array with \texttt{pattern=iso} and vertical polarization. \\
Wireless propagation & Radio-map orientation $(0,0,0)$. Line-of-sight, specular reflection, and diffraction are enabled. All buildings use \texttt{itu\_concrete}, and the ground uses \texttt{itu\_very\_dry\_ground}; material parameters are fixed across all datasets. \\
Predictor & A compact version of the PMNet encoder--decoder architecture~\cite{Lee2023PMNet}. The encoder uses a $16$-channel stem and four bottleneck residual stages with $(1,1,3,1)$ blocks and output widths $(64,128,128,256)$. Atrous spatial pyramid pooling (ASPP) is used to combine five parallel $64$-channel branches and projects into $128$ channels. Six skip-connected decoder blocks with widths $(128,128,64,64,64,32)$ reconstruct the $128\times128$ map; the final head has widths $(34,16,16,1)$ and a sigmoid output. \\
Training & AdamW with learning rate $3\times10^{-4}$, weight decay $10^{-5}$, batch size $16$, and $30$ epochs; a ReduceLROnPlateau scheduler is adopted with factor $0.25$ and patience $3$. All models share the same initialization and training seed. \\
\bottomrule
\end{tabular}
\end{table}

\end{appendices}

\bibliographystyle{IEEEtran}
\bibliography{references}

\end{document}